\documentclass{article}

  \PassOptionsToPackage{numbers, compress}{natbib}

\usepackage[preprint]{neurips_2019}
\usepackage{lipsum}
\usepackage{enumitem}
\usepackage[utf8]{inputenc} 
\usepackage[T1]{fontenc}    
\usepackage{hyperref}       
\usepackage{url}            
\usepackage{booktabs}       
\usepackage{amsfonts}       
\usepackage{nicefrac}       
\usepackage{microtype}      
\usepackage{mathtools}
\usepackage{multirow}  
\usepackage{xcolor}
\usepackage{dsfont}
\usepackage{multicol}
\usepackage{amsmath,amssymb,amsfonts,amsthm,bbm,graphicx}
\usepackage{algorithm}
\usepackage{algorithmic}
\usepackage{amsthm}
\usepackage{xspace}
\usepackage{color}
\usepackage{hyperref}
\hypersetup{colorlinks=true}
\definecolor{darkgreen}{RGB}{0,96,0}
\definecolor{darkred}{RGB}{255,0,0}
\hypersetup{linkcolor=darkgreen,citecolor=blue,urlcolor=black,filecolor=black}  
\newcommand{\indic}{\mathds{1}}
\newcommand{\EE}{\mathbb{E}}
\newcommand{\reals}{\mathbb{R}}
\newcommand{\Lip}{L}
\newcommand{\Lipstar}{L^*}
\newcommand{\Dev}{D}
\newcommand{\Devstar}{D^*}
\newcommand{\Sensitivity}{S^*}
\newcommand{\iid}{\text{i.i.d.}}
\newcommand{\Var}{\mathrm{Var}}
\newcommand{\Paren}[1]{\left(#1\right)}
\newcommand{\Abs}[1]{\left|#1\right|}
\newcommand{\Br}[1]{\left[#1\right]}
\newcommand{\Brace}[1]{\left\{#1\right\}}
\newcommand{\upto}{,\ldots,}
\newcommand{\union}{\mathop{\cup}}

\newcommand{\degree}{{d_{n}}}
\newcommand{\PP}{\mathbb{P}}
\newcommand{\Poi}{\mathrm{Poi}}
\newtheorem{Theorem}{Theorem}

\newtheorem{theorem}{Theorem}
\newtheorem{lemma}[theorem]{Lemma}

\DeclareMathOperator*{\Exp}{\EE}

\title{Unified Sample-Optimal Property Estimation\\in Near-Linear Time}

\author{
  Yi~Hao\\
  Dept. of Electrical and Computer Engineering \\
 University of California, San Diego\\
  \texttt{yih179@ucsd.edu}
 \And
Alon~Orlitsky \\
 Dept. of  Electrical and Computer Engineering\\
 University of California, San Diego\\
  \texttt{alon@ucsd.edu}
}

\begin{document}

\maketitle

\begin{abstract}
We consider the fundamental learning problem of estimating properties
of distributions over large domains. Using a novel
piecewise-polynomial approximation technique, we derive the first
unified methodology for constructing sample- and time-efficient estimators for
all sufficiently smooth, symmetric and non-symmetric, additive
properties. This technique yields near-linear-time computable
estimators whose approximation values are asymptotically optimal and
highly-concentrated, resulting in the first: 1) estimators achieving the $\mathcal{O}(k/(\varepsilon^2\log k))$ 
min-max $\varepsilon$-error sample complexity for all $k$-symbol Lipschitz properties; 
2) unified near-optimal differentially private estimators for a variety of properties; 
3) unified estimator achieving optimal bias and near-optimal variance for five important properties; 
4) near-optimal sample-complexity estimators for several important symmetric properties over both domain sizes and confidence levels.
In addition, we establish a McDiarmid's inequality under Poisson sampling, which is of independent interest. 

\end{abstract}

\section{Introduction}
Let $\Delta_k$ be the collection of distributions over the alphabet
$[k] := \{1, 2, \ldots, k\}$, and $[k]^*$ be the set of finite
sequences over $[k]$. 
In many learning applications, we are given \iid\ samples
$X^n:= X_1, X_2, \ldots, X_n$ from an unknown distribution 
$\vec{p}:=(p_1, p_2, \ldots, p_k)\in \Delta_k$, and using these samples we would 
like to infer certain distribution properties.

A \emph{distribution property} is a mapping $f:\Delta_k\to \reals$.
Often, these properties are \emph{symmetric} and
\emph{additive}, namely, $f(\vec{p})=\sum_{i\in [k]}f(p_i)$. For example,
Shannon entropy, support size, and three more properties in Table~\ref{table1}. 
Many other important properties are 
additive but not necessarily symmetric,
namely, $f(\vec{p})=\sum_{i\in [k]}f_i(p_i)$.
For example, Kullback-Leibler (KL) divergence or $\ell_1$ distance \vspace{-0.1em}
to a fixed distribution $q$, and distances weighted by the observations $x_i$,\vspace{-0.05em}
e.g., $\sum_{i\in [k]}x_i\cdot|p_i-q_i|$. 
\vspace{-0.2em} 
A \emph{property estimator} 
is a mapping $\hat{f}: [k]^* \to
\mathbb{R}$, where $\hat{f}(X^n)$ approximates $f(\vec{p})$.

Property estimation has attracted significant attention due to its many applications in various disciplines:
Shannon entropy is the principal information measure in numerous machine-learning~\citep{chowliu} and neurosicence~\citep{snm} algorithms;
support size is essential in population~\citep{ca84} and vocabulary size~\citep{te87} estimation;
support coverage arises in ecological~\citep{ecological}, genomic~\citep{klr99}, and database~\citep{database} studies; 
$\ell_1$ distance is useful in hypothesis testing~\citep{hyptest} and classification~\citep{classification};  
KL divergence reflects the performance of
investment~\citep{investment}, compression~\citep{info}, and on-line
learning~\citep{onlinelearning}. 

For data containing sensitive information, we may need to design 
special property estimators that preserve individual privacy.
Perhaps the most notable notion of privacy is
\emph{differential privacy (DP)}. In the context of property
estimation~\citep{private}, we say that an estimator $\hat{f}$ is
\emph{$\alpha$-differentially private} if for any $X^n$ and $Y^n$ that
differ by at most one symbol, $\Pr(\hat{f}(X^n)\in
S)/\Pr(\hat{f}(Y^n)\in S)\leq e^\alpha$ for any measurable set
$S\subset{\mathbb{R}}$. We consider designing property
estimators that achieve small estimation error~$\varepsilon$, with
probability at least $2/3$, while maintaining
$\alpha$-privacy.

The next section formalizes our discussion and presents some of
the major results in the area. 

\subsection{Problem formulation and prior results}
\paragraph{Property estimation} Let $f$ be a property over $\Delta_k$. 
The $(\varepsilon,\delta)$-\emph{sample complexity of an estimator
$\hat{f}$} for $f$
is the smallest number of samples required to estimate $f(\vec{p})$
with accuracy $\varepsilon$ and confidence $1-\delta$, for all
distributions in $\Delta_k$,
\[
C_f(\hat{f},\varepsilon, \delta):=\min\{n: \Pr_{X^n\sim p}(|\hat{f}(X^n)-f(\vec{p})|> \varepsilon)
\leq \delta,\ \forall \vec p\in \Delta_k\}.
\]

The $(\varepsilon, \delta)$-\emph{sample complexity} of estimating $f$
is the lowest $(\varepsilon,\delta)$-sample complexity of any estimator,
\[
C_f(\varepsilon,\delta):=\min{}_{\!\hat{f}}\ C_f(\hat{f},\varepsilon,\delta).
\]
Ignoring constant factors and assuming $k$ is large, Table~\ref{table1} summarizes some of 
the previous results~\cite{ ICML,  jiaol1, alonunseen, VV11, improveentro,yihongsupp, yihongentro} 
for $\delta={1}/{3}$. Following the formulation in~\cite{ ICML, VV11, yihongsupp}, for support size, 
we normalize it by $k$ and replace $\Delta_k$ by the collection of distributions $\vec{p}\in \Delta_k$ 
satisfying $p_i\geq\frac1k, \forall i\in[k]$. For support coverage~\cite{ICML, alonunseen}, 
the expected number of distinct symbols in $m$ samples, we normalize it by the 
given parameter $m$ and assume that $m$ is sufficiently large.
\begin{table}[h]
  \caption{$C_f(\varepsilon, 1/3)$ for some properties}
  \label{table1}
  \centering
  \begin{tabular}{llll}
    \toprule
    Property & $f_i(p_i)$ & $C_f(\varepsilon, {1}/{3})$\\ 
    \midrule
    Shannon entropy & $p_i\log \frac{1}{p_i}$& ${\frac{k}{\log k}\frac{1}{\varepsilon}}$ \\ 
    Power sum of order $a$ & $p_i^a, a<1$ &  ${\frac{k^{{1}/{a}}}{\varepsilon^{{1}/{a}} \log k}}$\\ 
    Distance to uniformity &  $\left|p_i-\frac{1}{k}\right|$ & ${\frac{k}{\log k}\frac{1}{\varepsilon^2}}$ \\ 
    Normalized support size  & $\frac{\indic_{p_i>0}}{k}$ & ${\frac{k}{\log k}{\log^2{\frac{1}{\varepsilon}}}}$ \\ 
    Normalized support coverage & $\frac{1-(1-p_i)^m}{m}$ & ${\frac{m}{\log m}\log\frac{1}{\varepsilon}}$ \\ 
    \bottomrule
  \end{tabular}
\end{table}

\paragraph{Min-max MSE} A closely related characterization of an estimator's performance is
the \emph{min-max mean squared error (MSE)}.
For any unknown distribution $\vec{p}\in\Delta_k$, the MSE of a property
estimator $\hat{f}$ in estimating $f(\vec{p})$, using $n$ samples
from $\vec{p}$, is 
\[
R_n(\hat{f},f, \vec{p}):=\Exp_{X^n\sim \vec{p}}(\hat{f}(X^n)-f(\vec{p}))^2.
\]
Since $\vec{p}$ is unknown, 
we consider the minimal possible worst-case MSE, or the \emph{min-max MSE}, for any property estimator in estimating property $f$,
\[
R_n(f, \Delta_k):=\min_{\hat{f}}\max_{\!\vec p\in \Delta_k} R_n(\hat{f},f, \vec{p}).
\]\par\vspace{-0.8em}
The property estimator $\hat{f}^{\text{m}}$ achieving the min-max MSE is
the \emph{min-max estimator}~\cite{ jiaoentro, jiaol1, yihongsupp, yihongentro}. 

Letting $\vec{p}_{\text{max}}:=\arg\max_{\vec{p}\in\Delta_k}R_n(\hat{f}^{\text{m}},f, \vec{p})$
be the worst-case distribution for $\hat{f}^{\text{m}}$,
we can express the min-max MSE as the sum
of two quantities: the \emph{min-max squared bias},
\[
\text{Bias}_n^2(\hat{f}^{\text{m}}):=(\EE_{X^n\sim \vec{p}_{\text{max}}}[\hat{f}^{\text{m}}(X^n)]-f(\vec{p}_{\text{max}}))^2,
\] 
and the \emph{min-max variance},
\[
\Var_n(\hat{f}^{\text{m}}):=\Var_{X^n\sim \vec{p}_{\text{max}}}(\hat{f}^{\text{m}}(X^n)).
\]

\paragraph{Private property estimation} Analogous to the non-private setting above, 
for an estimator $\hat{f}$ of $f$, let its
$(\varepsilon,\delta,\alpha)$-\emph{private sample complexity}
$C(\hat f,\varepsilon,\delta,\alpha)$ 
be the smallest number of samples that $\hat f$ requires to estimate $f(\vec{p})$
with accuracy $\varepsilon$ and confidence $1-\delta$, 
while maintaining $\alpha$-privacy for all
distributions $\vec p\in\Delta_k$.
The $(\varepsilon, \delta,\alpha)$-\emph{private sample complexity} of estimating $f$
is then
\[
C_f(\varepsilon,\delta,\alpha):=\min{}_{\!\hat{f}}\ C_f(\hat{f},\varepsilon,\delta,\alpha).
\]
For Shannon entropy, normalized support size,
 and normalized support coverage, the work of~\cite{priv_prop}
derived tight lower and upper bounds on $C_f(\varepsilon,{1}/{3}, \alpha)$.

\subsection{Existing methods}

There are mainly two types of methods introduced to estimate distribution properties:
plug-in, and approximation-empirical, which we briefly discuss below. 

\paragraph{Plug-in} Major existing plug-in estimators work for only symmetric properties,
and in general do not achieve the min-max MSEs' nor the optimal
$(\varepsilon,\delta)$-sample complexities. 
More specifically, 
the linear-programming based
methods initiated by~\cite{efron}, and
analyzed and extended in~\cite{VV11, powerlinear,  improveentro} achieve
the optimal sample complexities only for distance to uniformity and
entropy, for relatively large $\varepsilon$. The method
 basically learns the moments of the underlying distribution from its samples, and finds
a distribution whose (low-order) moments are consistent with these estimates.
A locally refined version of the linear-programming estimator~\cite{JiaoLLM} 
achieves optimal sample complexities for entropy, power sum, and 
normalized support size, but requires
polynomial time to be computed. This version yields a bias guarantee 
similar to ours over symmetric properties, yet its variance 
guarantee is often worse. 

Recently, the work of~\cite{ICML} showed that the profile maximum likelihood (PML)
estimator~\cite{OSVZ04}, an estimator that finds a distribution maximizing
 the probability of observing the multiset of empirical frequencies,
is sample-optimal for estimating entropy, distance to uniformity, 
and normalized support size and coverage. After the initial submission of the
current work, paper~\cite{HPML} showed that the PML approach and
its near-linear-time computable variant~\cite{C19} are
sample-optimal for any property that is symmetric, 
additive, and appropriately Lipschitz, including the four properties just mentioned.
This~establishes the PML estimator as the first universally sample-optimal plug-in approach 
for estimating symmetric properties. In comparison, the current work provides a unified 
property-dependent approach that is sample-optimal for several symmetric and 
non-symmetric properties. 

\paragraph{Approximation-empirical} The approximation-empirical method~\cite{jiaoentro,
  jiaol1, Paninski03,  yihongsupp, yihongentro} identifies a non-smooth part of the
underlying function $f_i$, replaces it 
by a polynomial $\tilde{f}_i$, and estimates the value of $p_i$
by its empirical frequency $\hat{p}_i$.  Depending on whether $\hat
p_i$ belongs to the non-smooth part or not, the method estimates
$f_i(p_i)$ by either the unbiased estimator of $\tilde{f}_i(p_i)$, or
the empirical plug-in estimator $f_i(\hat{p}_i)$. However, 
due to its strong dependency on both the function's
structure and the empirical estimator's performance, the method requires
significantly different case-by-case modification and analysis, and
may not work optimally for general additive properties. 

Specifically, 1) The efficacy of this method relies on the accuracy of the
empirical plug-in estimator over the smooth segments, which needs to
be verified individually for each property; 
2) Different functions often have non-smooth segments of different number, locations, and sizes; 
3) Combining the non-smooth and smooth segments estimators requires additional care: 
sometimes needs the knowledge of $k$, sometimes even needs a third estimator to ensure smooth transitioning. 

In addition, the method has also not been shown to achieve optimal results for general
Lipschitz properties, or many of the other properties covered by the 
new method in this paper.

\section{New methodology}
The preceding discussion showed that no existing generic
method efficiently estimates general additive properties. 
Motivated by recent advances in the field~\cite{ICML, JiaoLLM, H19, H18}, 
we derive the first generic method to construct sample-efficient
estimators for all sufficiently smooth additive properties. 

We start by approximating functions
of an unknown Bernoulli success probability from its \iid\ samples.
For a wide class of real functions, we propose a
piecewise-polynomial approximation technique,
and show that it yields small-bias estimators that are 
exponentially concentrated around their expected estimates. 
This provides a different view of property estimation that allows us 
to simplify the proofs and broaden the range of the results. 
For details please see Section~\ref{sec:bern_functions}. 

\paragraph{High-level idea} 
The idea behind this methodology is natural. By the Chernoff bound for binomial random variables, 
the empirical count of a symbol in a given sample sequence will not differ from its  
mean value by too much. Hence, based on the empirical frequency, we can roughly infer 
which ``tiny piece'' of $[0,1]$ the corresponding probability lies in. However, due to randomness, a symbol's empirical frequency may often differ from 
 the true probability value by a small quantity, and plugging it into the function will cause certain amount of bias. 
 
 To correct this bias,
 we first replace the function by its low-degree polynomial approximation \emph{over that ``tiny piece''}, and then compute an unbiased estimator of this polynomial. 
 In other words, we use this polynomial as a proxy for the estimation task. We want the degree of the polynomial to be small since this will generally reduce the unbiased estimator's variance; we focus on approximating only a tiny piece of the function because this will reduce the polynomial's approximation error (bias). Given any additive property $f(\vec{p})=\sum_{i\in[k]} f_i(p_i)$, we apply\vspace{-0.15em}
this technique to each \vspace{-0.1em}real function $f_i$ and use the corresponding sum to estimate
$f(\vec{p})$.  Henceforth we use $\hat{f}^*$ to denote this explicit estimator.  \vspace{-0.4em}

\section{Implications and new results}\label{sec:result}\vspace{-0.5em}
Because of its conceptual simplicity, the methodology described in the last section has strong
implications for estimating all sufficiently smooth additive properties, which we present as theorems. 

{\bf Theorem~\ref{cor:sample_guarantees}} in Section~\ref{sec:prop} is the root of all the following results,
while it is more abstract and illustrating it requires much more effort. Hence for clarity, we begin by presenting several more concrete results. \vspace{-0.5em}

\paragraph{Correct asymptotic} For most of the properties considered in the paper, even the naive empirical-frequency estimator is sample-optimal in the large-sample regime (termed "simple regime" in~\cite{yihongsupp}) where the number of samples $n$ far exceeds the alphabet size $k$. The interesting regime, addressed in numerous recent publications~\cite{JiaoLLM, H19, H18,  HPML, jiaoentro, VV11, improveentro, yihongentro}, is where $n$ and $k$ are comparable, e.g., differing by at most a logarithmic factor. In this range, $n$ is sufficiently small that sophisticated techniques can help, yet not too small that nothing can be estimated. Since $n$ and $k$ are given, one can decide whether the naive estimator suffices, or sophisticated estimators are needed. For most of the results presented here, the technical significance stems in their nontriviality in this large-alphabet regime. 
For this reason, we will also assume that $\log k\asymp\log n$ throughout the paper. 

\subsection*{Implication 1: Lipschitz property estimation}\vspace{-0.3em}

 An additive property $f(\vec{p})=\sum_{i} f_i(p_i)$ is \emph{$L$-Lipschitz}\vspace{-0.05em} if all functions $f_i$ have
Lipschitz constants uniformly bounded by $L$. Many important properties are 
Lipschitz, but except for a few isolated examples, it was not known until very 
recently~\cite{H19, H18} that general 
Lipschitz properties can be estimated with sub-linearly many
samples. 
In particular, the result in~\cite{H19} implies a sample-complexity upper bound of 
$\mathcal{O}(L^3k/(\varepsilon^3\log k))$. We improve this bound to 
$C_f(\varepsilon,1/3)\lesssim {L^2k}/{(\varepsilon^2 \log k)}$: 
\begin{Theorem}\label{cor:Lipschitz}
If $f$ is an $L$-Lipschitz property, then for any $p\in\Delta_k$ and $X^n\sim p$,\vspace{-0.3em}
\[
\Abs{\EE\Br{\hat{f}^*(X^{n})}-f(\vec{p})}\lesssim\sum_{i\in[k]}L\sqrt{\frac{p_i}{n\log n}}\leq L\sqrt{\frac{k}{n\log n}},\vspace{-0.3em}
\]
and
\[
\Var(\hat{f}^*(X^{n}))\leq \mathcal{O}\Paren{\frac{L^2}{n^{0.99}}}.\vspace{-0.4em}
\]
\end{Theorem}
This theorem is optimal as even for relatively
simple Lipschitz properties, e.g., distance to uniformity (see Table~\ref{table1} and~\cite{jiaol1}),  the bias bound is optimal up to constant factors, and the variance bound is near-optimal and can not be smaller than $\Omega(L^2/n)$.  

\subsection*{Implication 2: High-confidence property estimation}

Surprisingly, the $(\varepsilon,\delta)$-sample complexity has not been
fully characterized even for some important properties. A commonly-used
approach to constructing an estimator with $(\varepsilon,\delta)$-guarantee
is to choose an $(\varepsilon,{1}/{3})$-estimator, and boost the learning confidence
by taking the median of its $\mathcal{O}(\log\frac{1}{\delta})$
independent estimates. This well-known \emph{median trick} yields the following upper bound
\[
C_f(\varepsilon,\delta)\lesssim \log\frac{1}{\delta}\cdot C_f(\varepsilon,{1}/{3}). \vspace{-0.25em}
\]
For example, for Shannon entropy, 
\[
C_f(\varepsilon,\delta)\lesssim \log\frac{1}{\delta}\cdot{\frac{k}{\varepsilon \log k}+\log\frac{1}{\delta}\cdot\frac{\log^2 k}{\varepsilon^2}}. \vspace{-0.1em}
\]
By contrast, we show that our estimator satisfies
\[
 C_f(\hat{f}^*,\varepsilon,\delta)\lesssim \frac{k}{\varepsilon \log k} +
 \Paren{\log \frac{1}{\delta}\cdot\frac{1}{\varepsilon^2}}^{1.01}.
\]
To see optimality, Theorem~\ref{cor:highprob} below shows that this upper bound is nearly tight. 

In the high-probability regime, namely when $\delta$ is small, the new
upper bound obtained using our method could be significantly smaller than the one obtained from the 
median trick. 
Theorem~\ref{cor:highprob} shows that this phenomenon also holds for other
properties like normalized support size and power sum.

\begin{Theorem}\label{cor:highprob}
Ignoring constant factors, Table~\ref{table3} summarizes relatively tight lower
and upper bounds on $C_f(\varepsilon,\delta, k)$ for different
properties. In addition, all the upper bounds can be achieved by 
estimator~$\hat{f}^*$. 
\end{Theorem}

\begin{table}[ht]
  \caption{Bounds on $C_f(\varepsilon,\delta, k)$ for several properties}
  \label{table3}
  \centering
  \begin{tabular}{llll}
    \toprule
    Property  & Lower bound & Upper bound \\ 
    \midrule
    Shannon entropy & ${\frac{k}{\varepsilon \log k}+\log \frac{1}{\delta}\cdot\frac{\log^2 k}{\varepsilon^2}}$&$\frac{k}{\varepsilon \log k}+\Paren{\log \frac{1}{\delta}\cdot\frac{1}{\varepsilon^2}}^{1+\beta}$\\ 
    Power sum of order $a$ &  $\frac{k^{\frac{1}{a}}}{\varepsilon^{\frac{1}{a}} \log k}+\log \frac{1}{\delta}\cdot\frac{k^{2-2a}}{\varepsilon^2}$ & $\frac{k^{\frac{1}{a}}}{\varepsilon^{\frac{1}{a}} \log k}+\Br{\Paren{\log \frac{1}{\delta}\cdot\frac{1}{\varepsilon^2}}^{\frac{1}{2a-1}}}^{1+\beta}$\\
    Normalized support size   & $\frac{k}{\log k}\log^2\frac{1}{\varepsilon}$&$\frac{k}{\log k}\log^2\frac{1}{\varepsilon}$\\ 
    \bottomrule
  \end{tabular}
\end{table}

\vspace{0.3em}
\emph{Remarks on Table~\ref{table3}}: Parameter $\beta$ can be any fixed absolute constant in $(0,1)$. The lower and upper bounds for power sum hold for $a\in({1}/{2},1)$. For normalized support size,
we require $\delta>\exp(-k^{{1}/{3}})$ and $\varepsilon\in(k^{-0.33},
k^{-0.01})$. Note that the restriction on $\varepsilon$ for support-size estimation is imposed only to 
yield a simple sample-complexity expression. 
This is not required by our algorithm, which is also sample optimal for $\varepsilon\ge k^{-0.01}$.
It is possible that other algorithms can achieve similar upper bounds, while our main point is to demonstrate that our single, unified method has many desired attributes.

\subsection*{Implication 3: Optimal bias and near-optimal variance}

The min-max MSEs of several important properties have been determined\vspace{-0.1em}
up to constant factors, yet there is no explicit and executable 
scheme for designing estimators achieving them. 
We show that $\hat{f}^*$ achieves optimal squared bias 
and near-optimal variance in estimating a variety of properties. 
\begin{Theorem}\label{cor:mse_optimal}
Up to constant factors, the estimator $\hat{f}^*$ achieves the optimal (min-max) 
squared bias and near-optimal variance  
for estimating Shannon entropy, normalized
support size, distance to uniformity, and power sum, and $\ell_1$ distance 
to a fixed~distribution.
\end{Theorem}
\emph{Remarks on Theorem~\ref{cor:mse_optimal}}: 
For power sum, we consider the case where the order is less than $1$. 
For normalized support size, we again make the assumption that the 
minimum nonzero probability of the underlying distribution is at least $1/k$.
As noted previously, we consider the parameter regime where 
$n$ and $k$ are comparable and $k$ is large. 
In particular, besides the general assumption $\log k\asymp \log n$, we assume that $n\gtrsim k^{1/\alpha}/\log k$ for power sum; $n\gtrsim k/\log k$ for entropy; 
and $k\log k\gtrsim n\gtrsim k/\log k$ for normalized support size. 
The proof of the theorem naturally follows from Theorem~\ref{cor:sample_guarantees}.

\subsection*{Implication 4: Private property estimation}

 Privacy is of increasing concern in modern data science.
We show that our estimates are exponentially concentrated around the
underlying value. Using this attribute we derive a near-optimal
differentially-private estimator $\hat f^*_{\textit{DP}}$ for several important properties
by adding Laplacian noises to $\hat f$.

As an example, for Shannon entropy and properly chosen algorithm hyper-parameters, 
\[
C_f(\hat f^*_{\textit{DP}},\varepsilon,1/3,\alpha)
\lesssim \frac{k}{\varepsilon\log k}+\frac{1}{\varepsilon^{2.01}}
+\frac{1}{{(\alpha\varepsilon)}^{1.01}}.
\]
This essentially recovers the sample-complexity upper bound in~\cite{priv_prop}, which is nearly tight~\cite{priv_prop} for all parameters. 
Hence for large domains, one can achieve strong differential
privacy guarantees with only a marginal increase in the sample size\vspace{-0.15em}
 compared to the $k/(\varepsilon\log k)$ required for non-private estimation. An analogous argument shows that $\hat f^*_{\textit{DP}}$ is also near-optimal for the private estimation of support size and many others. 
 Appendix~\ref{sec:dp} provides more detail and unified bounds on the differentially-private sample complexities of general additive properties.

\paragraph{Outline} 
The rest of the paper is organized as follows. 
In Section~\ref{sec:bern_functions} we construct an estimator that approximates the function value of an unknown Bernoulli success probability, 
and characterize its performance by Theorem~\ref{thm:single_function_estimator}. 
In Section~\ref{sec:prop} we apply this function 
estimator to estimating properties of distributions
and provide analogous guarantees. 
Section~\ref{conclusion} concludes the paper and also presents possible future directions.
We postpone all the proof details to the appendices.

\section{Estimating functions of Bernoulli probabilities}\label{sec:bern_functions}\vspace{-0.1em}
\subsection{Problem formulation}\label{adddefinitions}
We begin with a simple problem that involves just a single unknown parameter. 

Let $g$ be a continuous real function over the unit interval whose absolute value
 is uniformly bounded by an absolute constant. Given \iid\ samples
$X^{n}:=X_1,\ldots,X_{n}$ from a Bernoulli distribution with unknown
success probability $p$, we would like to 
estimate the function value $g(p)$.
A \emph{function estimator} is a mapping $\hat{g}:\{0,1\}^*\to\mathbb{R}$. We characterize the performance of the estimator $\hat{g}{(X^n)}$ in estimating $g(p)$ by its \emph{absolute bias}
\[
\text{Bias}(\hat g):=|\EE[\hat g(X^n)]-g(p)|,
\]
and \emph{deviation probability}
\[
P(\varepsilon):=\Pr(|\hat g(X^n)-\EE[\hat g(X^n)]|>\varepsilon),
\]
which implies the variance, and provides additional
information useful for property estimation. 
Our objective is to find an estimator that has near-optimal small bias and
Gaussian-type deviation probability $\exp(-n^{\Theta(1)}\varepsilon^2)$
for all possible values of $p\in[0,1]$. 

As could be expected, our results are closely related
to the smoothness of the function $g$. 

\subsection{Smoothness of real functions}\label{smoothness}
\paragraph{Effective derivative} Given an interval $I$ and step size 
$h\in(0, |I|)$, where $|I|$ denotes the interval's length. 
The \emph{effective derivative} of $g$ over $I$
is the Lipschitz-type ratio 
\[
\Lip_{g}(h, I)
:=
\sup_{x,y\in I, |y-x|\ge h} \frac{\Abs{g\Paren{y}-g\Paren{x}}}{|y-x|}.
\]
This simple smoothness measure does not fully capture the smoothness 
of $g$. For example, $g$ could be a zigzag function that has a high
effective derivative locally, but over-all fluctuates in only a very
small range,
and hence is close to a smooth function in maximum deviation. 
We therefore define a second smoothness  measure as the maximum
deviation between $g$ and a fixed-degree polynomial.
Besides being smooth and having derivatives of all orders,
by the Weierstrass approximation theorem, 
polynomials can also uniformly approximate any continuous $g$.

\paragraph{Min-max deviation} Let $P_d$ be the collection of polynomials of degree at most $d$.
The \emph{min-max deviation} in approximating $g$ over 
an interval $I$ by a polynomial in $P_d$ is 
\[
\Dev_g(d, I)
:=
\min_{q\in P_d}\max_{x\in I}|g(x)-q(x)|.
\]
The minimizing polynomial is the degree-$d$ min-max polynomial
approximation of $g$ over $I$.

For simplicity we abbreviate $\Lip_{g}(h):=\Lip_{g}(h, [0,1])$ and
$\Dev_g(d):=\Dev_g(d, [0,1])$. 

\subsection{Estimator construction}\label{sec:single_function_estimator}
For simplicity, assume that the sampling parameter is an even number
$2n$. Given \iid\ samples $X^{2n}\sim \text{Bern}(p)$, we 
let $N_i$ denote the number of times symbol $i\in\{0,1\}$ appears in $X^{2n}$.

We first describe a simplified version of our estimator and 
provide a non-rigorous analysis relating its performance to
the smoothness quantities just defined. The actual more
involved estimator and a rigorous performance analysis are presented
in Appendix~\ref{sec:prop} and~\ref{cor2}.

\paragraph{High-level description} On a high level, the empirical estimator estimates $g(p)$ by $g(N_1/(2n))$, and often incurs a large bias. To address this, we first partition the unit interval into roughly $\sqrt{n}$ sub-intervals. Then, we split $X^{2n}$ into two halves of equal length $n$ and use the empirical probability of symbol $1$ in the first half to identify a sub-interval $I$ and its two neighbors in the partition so that $p$ is contained in one of them, with high confidence. Finally, we replace $g$ by a low-degree min-max polynomial $\tilde{g}$ over $I$ and its four neighbors and estimate $g(p)$ from the second half of the sample sequence by applying a near-unbiased estimator of $\tilde{g}(p)$.

\subsection*{Step 1: Partitioning the unit interval}
Let $\alpha[a,b]$ denote the interval $[\alpha a,\alpha b]$. 
For an absolute positive constant $c$,
define $c_{n}:= c\frac{\log n}{n}$ and a sequence of increasing-length intervals
\begin{align*}
I_j&:=c_{n} \left[(j-1)^2,j^2\right],\quad j\ge1.
\end{align*}
Observe that the first $M_{n}:=1/\sqrt{c_{n}}$ intervals
partition the unit interval $[0,1]$.
For any $x\ge 0$, we let $j_{x}$ denote the index $j$ such that $x\in I_j$.
This unit-interval partition is directly motivated by the Chernoff bounds. A very similar construction appears in~\cite{Acharya13}, and the exact one appears in~\cite{JiaoLLM, HAICML}.
\subsection*{Step 2: Splitting the sample sequence and locating the probability value} 
Split the sample sequence $X^{2n}$ into two equal halves, and
let $\hat{p}_{1}$ and $\hat{p}_{2}$ denote the empirical
probabilities of symbol $1$ in the first and second half,
respectively. 
By the Chernoff bound of binomial random variables,
for sufficiently large $c$, the intervals $I_1\upto I_{M_{n}}$
form essentially the finest partition of $[0,1]$ such that 
if we let $I^{*}_j:=\union_{j'=j-1}^{j+1} I_{j'}$ and
$I^{**}_j:=\union_{j'=j-2}^{j+2} I_{j'}$, then for all underlying $p\not\in I^{*}_j$,
\[
\Pr(\hat{p}_{1}\in I_j)\le n^{-3},
\]
and for all underlying $p$ and all $j$,
\[
\Pr(\hat{p}_{1}\in I_j\text{ and } \hat{p}_{2}\not\in I^{**}_j)\le n^{-3}.
\]
It follows that if $\hat{p}_{1}\in I_j$, then with high confidence we can
assume that $p\in I^{*}_j$. 

\subsection*{Step 3: Min-max polynomial estimation}\label{sec:min-max_polynomial}
Let $\lambda$ be a universal constant in $(0,1/4)$ that balances
the bias and variance of our estimator.
Given the sampling parameter $n$, define
\[
\degree:=\max\Brace{d\in{\mathbb{N}}:\ d\cdot2^{4.5d+2}\leq {n}^{\lambda}}.
\]
For each $j$, let the \emph{min-max polynomial} of $g$ 
be the degree-$\degree$ polynomial $\tilde{g}_j$ minimizing the largest
absolute deviation with $g$ over $I^{**}_j$. 

For each interval $I_j$ we create a piecewise polynomial 
$\tilde g^*_j$ that approximates $g$ over the entire unit interval. 
It consists of $\tilde{g}_j$ over $I^{**}_j$, and  of $\tilde{g}_{j'}$
over $I_{j'}$ for $j'\not\in[j-2,j+2]$.

Finally, to estimate $g(p)$, let $j$ be the index such that
$\hat p_{1}\in I_j$, and approximate $\tilde g^*_j(p)$ by 
plugging in unbiased estimators of $p^t$ constructed from $\hat{p}_{2}$ for all powers $t\le\degree$. 
Note that a standard unbiased estimator for $p^t$ is $\prod_{i=0}^{t-1}[(\hat{p}_{2}-i/n)/(1-i/n)]$, and the rest follows from the linearity of expectation. 

\paragraph{Computational complexity}
A well-known approximation theory result states that 
the degree-$d$ truncated Chebyshev series (or polynomial)
of a function $g$, often closely approximate the degree-$d$ min-max polynomial of $g$.
The Remez algorithm~\citep{chebfun, chebfunbook} is a popular method
for finding such Chebyshev-type approximations of degree $d$, and is often 
very efficient in practice. 
Under certain conditions on the function to approximate, 
running the algorithm for $\log t$ iterations will lead to an error of $\mathcal{O}(\exp(-\Theta(t)))$.  
Indeed, many state-of-the-art \emph{property} estimators,
e.g.,~\citep{ H19, jiaoentro,jiaol1, yihongsupp,yihongentro}, use the Remez
algorithm to approximate the min-max polynomials, 
and have implementations that are near-linear-time computable.\vspace{-0.5em}
\subsection{Final estimator and its characterization}\label{sec:final_estimator}\vspace{-0.3em}
\paragraph{The estimator} Consolidating above results, we estimate $g(p)$ by the estimator 
\[
\hat{g}(\hat{p}_{1},\hat{p}_{2}):=\sum{}_{j}\ \hat{g}_{j}(\hat{p}_{2})\cdot \indic_{\hat{p}_{1}\in I_j}.\vspace{0.1em}
\]
The exact form and construction of this estimator appear in Appendix~\ref{sec:estconst}. 
\paragraph{Characterization} 
The theorem below characterizes the bias, variance, 
and mean-deviation probability of the estimator. 
We sketch its proof here and leave the details to the appendices.

According to the reasoning in the last section, for all $p\in I^*_j$, the absolute bias of the resulting estimator $\hat{g}_j(\hat{p}_{2})$ in estimating $g(p)$ is essentially upper
bounded by $\Dev_g(\degree, I^{**}_j)$. Normalizing it by the input's
precision $1/n$, we define the (normalized) \emph{ local min-max deviation} and the \emph{global min-max deviation} over
$I^{**}_j$, respectively, as 
\[
\Devstar_{g}(2n,x)
:=
n\cdot\!\!\! \max_{j'\in[j_x-1,j_x+1]}\Dev_g(\degree, I^{**}_{j'}).
\]
and \vspace{-0.5em}
\[
\Devstar_{g}(2n)
:=
\max_{x'\in[0,1]} \Devstar_{g}(2n,x').
\]\par \vspace{-0.5em}

Hence the bias of $\hat{g}(\hat{p}_{1},\hat{p}_{2})$ in estimating $g(p)$ is essentially upper bounded by
${\Devstar_{g}(2n,{p})}/{n}\leq {\Devstar_{g}(2n)}/{n}$.
A similar argument yields the following variance bound on $\hat{g}(\hat{p}_{1},\hat{p}_{2})$, where $\Devstar_{g}(2n,{p})$ is replaced by the \emph{local effective derivative},
\[
\Lipstar_{g}(2n, {p}):=\max_{j'\in[j_{p}-1,j_{p}+1]}\Lip_{g}(1/n,I^{**}_{j'}).
\]
 Analogously, define $\Lipstar_{g}(2n):=\max_{x\in[0,1]}\Lipstar_{g}(2n, x)$ as the \emph{global effective derivative}. The mean-deviation probability of this estimator is characterized by 
\[
\Sensitivity_g(2n):= \Lipstar_{g}(2n)+\Devstar_{g}(2n).
\]
Specifically, changing one sample in $X^{2n}$ changes the value of $\hat{g}(\hat{p}_{1},\hat{p}_{2})$ by at most $\Theta{(\Sensitivity_g(n)n^{\lambda-1})}$.
Therefore, by McDiarmid's inequality, for any error parameter $\varepsilon$,
\[
\Pr(|\hat{g}(\hat{p}_{1},\hat{p}_{2})-\EE[\hat{g}(\hat{p}_{1},\hat{p}_{2})]|>\varepsilon)
\lesssim \exp\Paren{-\Theta\Paren{\frac{\varepsilon^2 n^{1-2\lambda}}{\Sensitivity_g(2n)^2}}}. 
\]

\begin{Theorem}\label{thm:single_function_estimator}
For any bounded function $g$ over $[0,1]$, $X^{n}\sim \text{Bern}(p)$, and error parameter $\varepsilon>0$, 
\[
\Abs{\EE[\hat{g}(\hat{p}_{1},\hat{p}_{2})]-g(p)}\lesssim\frac{p}{n^3}+\frac{\Devstar_{g}(n,{p})}{n},
\]
\[
\Var(\hat{g}(\hat{p}_{1},\hat{p}_{2}))\lesssim \frac{p}{n^5} + \frac{\Paren{ \Lipstar_{g}(n, {p})}^2\cdot p}{n^{1-4\lambda}},\vspace{-0.5em} 
\]
and
\[
\Pr(|\hat{g}(\hat{p}_{1},\hat{p}_{2})-\EE[\hat{g}(\hat{p}_{1},\hat{p}_{2})]|>\varepsilon)
\lesssim \exp\Paren{-\Theta\Paren{\frac{\varepsilon^2 n^{1-2\lambda}}{\Sensitivity_g(n)^2}}}.
\]
\end{Theorem}\vspace{-0.2em}

Next we use this theorem to derive tight bounds for
estimating general additive properties.

\section{A unified piecewise-polynomial approach to property estimation}\label{sec:prop}\vspace{-0.1em}

Let $f$ be an arbitrary additive property over $\Delta_k$
such that $|f_i(x)|$ is uniformly bounded by some absolute constant 
for all $i\in[k]$, and $L^*_{ \boldsymbol\cdot}(\cdot)$, $\Devstar_{
  \boldsymbol\cdot}(\cdot)$, and $S^*_{ \boldsymbol\cdot}(\cdot)$ be
the smoothness quantities defined in Section~\ref{sec:min-max_polynomial} and~\ref{sec:final_estimator}.
Let $X^n$ be an \iid\ sample sequence from an unknown distribution 
$\vec{p}\in \Delta_k$. 
Splitting $X^n$ into two sub-sample sequences of equal length, we denote by $\hat{p}_{i,1}$ and $\hat{p}_{i,2}$ the empirical probability of symbol $i\in[k]$ in the first and second sub-sample sequences, respectively.

Applying the technique presented\vspace{-0.15em} in Section~\ref{sec:bern_functions}, we can estimate the additive property
$f(\vec{p})=\sum_{i\in [k]} f_i(p_i)$
by the estimator\vspace{-0.25em}
$
\hat{f}^*(X^n):=\sum_{i\in[k]} \hat{f}_i(\hat{p}_{i,1},\hat{p}_{i,2}).
$
Theorem~\ref{thm:single_function_estimator} can then be used to show
that $\hat{f}^*$ performs well for all sufficiently-smooth additive properties:

\begin{Theorem}\label{cor:sample_guarantees}
For any $\varepsilon>0$, $\vec{p}\in \Delta_k$, and $X^{n}\sim \vec{p}$,
\[
\Abs{\EE\Br{\hat{f}^*(X^{n})}-f(\vec{p})}\lesssim\frac{1}{n^3}+\frac{1}{n}\sum_{i\in[k]}\Devstar_{f_i}(n,{p_i}),
\]\par\vspace{-0.5em}
\[
\Var(\hat{f}^*(X^{n}))\lesssim \frac{1}{n^5} + \frac{1}{n^{1-4\lambda}}\sum_{i\in[k]}\Paren{ \Lipstar_{f_i}(n,{p_i})}^2\cdot p_i,
\]\par\vspace{-0.75em}
and
\[
\Pr\Paren{\Abs{\hat{f}^*(X^{n})-\EE\Br{\hat{f}^*(X^{n})}}>\varepsilon}
\lesssim \exp\Paren{-\Theta\Paren{\frac{\varepsilon^2 n^{1-2\lambda}}{{\max}_{i\in[k]}(\Sensitivity_{f_i}(n))^2}}}.
\]
\end{Theorem} 
\paragraph{Discussions} While the significance of the theorem may not be immediately apparent,
note that the three equations characterize the estimator's bias,
variance, and higher-order moments in terms of the
local min-max deviation $\Devstar_{f_i}(n,{p_i})$, 
the local effective deviation $\Lipstar_{f_i}(n,{p_i})$,
and the sum of the maximum possible values of the two, $\Sensitivity_{f_i}(n)$, respectively.
The smoother function $f_i$ is, the smaller $\Devstar_{f_i}(\cdot)$ and $\Lipstar_{f_i}(\cdot)$ will be.
In particular, for simple smooth functions, the values of $\Devstar$, $\Lipstar$, and $\Sensitivity$ can be easily shown to be
small, implying that the $f^*$ is nearly optimal under all three criteria.

For example, considering Shannon entropy where $f_i(p_i)=-p_i\log p_i$ for all $i$, we can show that $\Devstar_{f_i}(n,{p_i})$ and $\Lipstar_{f_i}(n,{p_i})$ are \vspace{-0.1em} at most $\mathcal{O}(1/\log n)$ and $\mathcal{O}(\log n)$, respectively. Hence, the bias and variance bounds in Theorem~\ref{cor:sample_guarantees} become $k/(n\log n)$ and $(\log n)/n^{1-4\lambda}$, and the tail bound simplifies to $\exp(-\Theta(\varepsilon^2 n^{1-2\lambda}/\log^2 n))$, where $\lambda$ is an arbitrary absolute constant in $(0,1/4)$, e.g., $\lambda=0.01$. According to Theorem~\ref{cor:highprob} and results in~\citep{jiaoentro,yihongentro}, all these bounds are optimal.

\paragraph{Computational complexity}  We briefly illustrate how our
estimator can be computed efficiently in near-linear time in the
sample size~$n$. As stated in Section~\ref{sec:min-max_polynomial},
over each of the $O(\sqrt{n/\log n})$-intervals we constructed, we
will find the min-max polynomial of the underlying function over that
particular interval, and for many properties, an approximation suffices and 
the computation takes only $\text{poly} (\log n)$ time utilizing the 
Remez algorithm as previously noted.

Though our construction uses $O(\sqrt{n/\log n})$ such polynomials, for each symbol $i$ \emph{appearing} in the sample sequence $X^n$, we need to compute just one such polynomial to estimate $f_i(p_i)$. The number of symbols appearing in $X^n$ is trivially at most $n$, hence the total time complexity is $O(n\cdot \text{poly} (\log n))$, which is near-linear in $n$. 
In fact, the computation of all the 
$O(k\sqrt{n/\log n})$ possible
polynomials can be even performed off-line (without samples), 
and will not affect our estimator's time complexity.

\section{Conclusion and future directions}\label{conclusion}\vspace{-0.5em}
We introduced a piecewise min-max polynomial methodology for
approximating additive distribution properties. This method yields
the first generic approach to constructing sample- and time-efficient
estimators for all sufficiently smooth properties.
This approach provides the first:
1) sublinear-sample estimators for general Lipschitz properties;
2) general near-optimal private estimators;
3) unified min-max-MSE-achieving estimators for six important
properties;
4) near-optimal high-confidence estimators.
Unlike previous works, our method covers both symmetric and
non-symmetric, differentiable and non-differentiable, properties,
under both private and non-private settings. 

In addition, in Appendix~\ref{conc}, we establish a McDiarmid's inequality under Poisson sampling, which is of independent interest.

Two natural extensions are of interest: 1) generalizing the results to
properties involving multiple unknown distributions such as 
distributional divergences; 2) extending the techniques to
derive a similarly unified approach for the closely related
field of distribution property testing.

Besides the results we established for piecewise polynomial estimators 
under the min-max estimation framework,
the works of~\cite{H19,H18} recently proposed and studied a different formulation
of competitive property estimation that aims to emulate the instance-by-instance performance of
the widely used empirical plug-in estimator, using a logarithmic smaller sample size. 
It is also quite meaningful to investigate the performance of our technique through this new formulation.
\subsubsection*{Acknowledgments}

We are grateful to the National Science Foundation (NSF) for supporting this work through grants CIF-1564355 and CIF-1619448.

\appendix
\section{Outline for appendices}
The appendices are organized as follows.

Appendix~\ref{sec:prop}: We address function estimation and proves Theorem~\ref{thm:single_function_estimator}. Our objective is to design a small-bias estimator whose approximation value is highly concentrated around its mean.

Appendix~\ref{sec:tail}: We present several ancillary results that will be used in subsequent proofs. 

Appendix~\ref{sec:estconst}: We construct the function estimator $\hat{g}^*$ using piecewise min-max polynomials.

Appendix~\ref{sec:bias},~\ref{sec:variance}, and~\ref{sec:conc}: We derive the bias, variance, and tail probability bounds presented in Theorem~\ref{thm:single_function_estimator}, respectively, showing that the estimator $\hat{g}^*$ admits strong theoretical guarantees for a broad class of functions. 

In particular, in Appendix~\ref{conc}, we establish a McDiarmid's inequality under Poisson sampling, which is of independent interest. 

Appendix~\ref{corproof}: We apply the function estimation technique derived in Section~\ref{sec:prop} to derive our generic method for learning additive properties, and prove other theorems. 

Appendix~\ref{cor2}: \vspace{-0.07em}We establish the results in Theorem~\ref{cor:sample_guarantees} and show that for all sufficiently smooth properties, our property estimator $\hat{f}^*$ achieves the state-of-the-art performance.

Appendix~\ref{cor5}: We consider the problem of estimating Lipschitz properties. By proving Theorem~\ref{cor:Lipschitz}, we show for the first time that all Lipschitz properties can be estimated
up to a small error $\varepsilon$ using $\mathcal{O}(k/(\varepsilon^2\log k))$ samples,
 with probability at least $2/3$.
 
Appendix~\ref{sec:dp}: We establish a general result on private property estimation, which trivially implies those stated in Section~\ref{sec:result} (Implication 4). 

Appendix~\ref{sec:highprob}: We utilize Theorem~\ref{cor:sample_guarantees} and some specific constructions to prove the upper and lower bounds in Theorem~\ref{cor:highprob}, respectively. 

\section{Proof of Theorem 4: Estimating functions of Bernoulli probabilities}\label{sec:prop}

\subsection{Ancillary results}\label{sec:tail}
\paragraph{Useful tools}\

The following two lemmas provide tight bounds on the tail probability of a Poisson or binomial random variable. We use these inequalities throughout the proofs.
\begin{lemma}[Chernoff Bound~\cite{concen}]\label{tailprob}
Let $X$ be a Poisson or binomial random variable with mean $\mu$, then for any $\delta>0$, 
\[
\PP(X\geq{(1+\delta)\mu})\leq{{\left(\frac{e^\delta}{(1+\delta)^{(1+\delta)}}\right)}^{\mu}}\leq{e^{-(\delta^2\land\delta)\mu/3}}
\]
and for any $\delta\in{(0, 1)}$, 
\[
\PP(X\leq{(1-\delta)\mu})\leq{{\left(\frac{e^{-\delta}}{(1-\delta)^{(1-\delta)}}\right)}^{\mu}}\leq{e^{-\delta^2\mu/2}}.
\]
\end{lemma}
By setting $\delta$ to be $1/2$ and $1$ in Lemma~\ref{tailprob}, we have the following corollary.
\begin{lemma}\label{cortail}
Let $X$ be a Poisson or binomial random variable with mean $\mu$, then
\[
\PP(X\leq{\frac{1}{2}\mu})\leq{e^{-0.15\mu}}
\]
and
\[
\PP(X\geq{2\mu})\leq{e^{-0.38\mu}}. 
\]
\end{lemma}

The \emph{$n$-sensitivity} of an estimator $\hat{f}$
is the maximum possible change in its value when a sample sequence of 
 size-$n$ input sequence is modified at exactly one location,
\[
S(\hat{f},n)
:=
\max\{|\hat{f}\Paren{x^n}-\hat{f}\Paren{y^n}|: x^n\text{ and
  }y^n\text{ differ in one location}\}.
\]
McDiarmid's inequality relates $S(\hat{f},n)$ to the tail probability of $\hat{f}(X^n)$.
\begin{lemma}[McDiarmid's inequality~\cite{McD}]\label{McD}
Let $\hat{f}$ be an estimator. For any constant $\varepsilon>0$, distribution $\vec p\in\Delta_k$, and \iid\ sample sequence $X^n\sim \vec p$,
\[
\Pr\Paren{\Abs{\hat{f}(X^n)-\EE[\hat{f}(X^n)]}>\varepsilon}\le 2\exp\Paren{-\frac{2\varepsilon^2}{nS^2(\hat{f},n)}}.
\]
\end{lemma}

As illustrated in the main paper, our construction relies on a variety of polynomials. To analyze these polynomials and relate them to other quantities, we often need to bound the polynomials' coefficients based on their ranges.
For a real polynomial, the next lemma provides tight upper bounds on the magnitude of its non-constant coefficients.
\begin{lemma}\label{coeffbound}
Let $p(x)=\sum_{j=0}^d a_j x^j$ be a degree-$d$ real polynomial and
\[
A:= \sup_{x_1,x_2\in[0,1]} |p(x_1)-p(x_2)|,
\]
then for $j\geq 1$,
\[
|a_j|\leq A \cdot 2^{3.5d}.
\]
\end{lemma}
We will utilize the above lemma to bound the variance of polynomial-based estimators.

\paragraph{Unbiased estimator of $\boldsymbol{(p-x)^v}$ and its characterization}\

The following polynomial is related to the unbiased estimator of $(p-x)^v$ under \emph{Poisson sampling}, where we make the sample size an independent Poisson random variable. 
Note that both $x\in\mathbb{R}$ and $v\in \mathbb{N}$ are given constant parameters. 
\[
h_{v,x}\Paren{y}:=\sum_{l=0}^{v}\binom{v}{l}(-x)^{v-l}\prod_{l'=0}^{l-1}\Paren{\frac{y}{n}-\frac{l'}{n}}.
\]
This polynomial will play an important role in our consecutive constructions and corresponding proofs. First, we establish and present several useful attributes of $h_{v,x}\Paren{y}$ below.
\begin{lemma}\label{boundg0}
For a Poisson random variable $Y\sim \Poi(np)$, 
\[
\EE[h_{v,x}(Y)]=(p-x)^v.
\]\par \vspace{-1em}
\end{lemma}

\begin{proof}
By the linearity of expectation and definition of Poisson random variables,
\begin{align*}
\EE[h_{v,x}(Y)]
& = \sum_{l=0}^{v}\binom{v}{l}(-x)^{v-l}\EE\Br{\prod_{l'=0}^{l-1}\Paren{\frac{Y}{n}-\frac{l'}{n}}}\\
& = \sum_{l=0}^{v}\binom{v}{l}(-x)^{v-l}\frac{1}{n^l}\EE\Br{\prod_{l'=0}^{l-1}\Paren{Y-l'}}\\
& = \sum_{l=0}^{v}\binom{v}{l}(-x)^{v-l}\frac{e^{-np}}{n^l}\sum_{j=0}^{\infty}{\frac{(np)^j}{j!}\prod_{l'=0}^{l-1}\Paren{j-l'}}\\
& = \sum_{l=0}^{v}\binom{v}{l}(-x)^{v-l}\frac{e^{-np}}{n^l}\sum_{j=l}^{\infty}{\frac{(np)^j}{(j-l)!}}\\
& = \sum_{l=0}^{v}\binom{v}{l}(-x)^{v-l}\frac{(np)^l}{n^l}\Paren{e^{-np}\sum_{j=l}^{\infty}{\frac{(np)^{j-l}}{(j-l)!}}}\\
& = \sum_{l=0}^{v}\binom{v}{l}(-x)^{v-l}p^l\\
& = (x-p)^v.\qedhere
\end{align*}
\end{proof}

Lemma~\ref{boundg0} implies that polynomial $h_{v,x}(Y)$ is the unbiased estimator of $(\EE[Y]/n-x)^v$ for $Y\sim \Poi(\cdot)$. 
The next three lemmas bound the polynomial's value when the input variable is close to its expectation.

\begin{lemma}\label{boundg1}
For a Poisson random variable $Y\sim \Poi(np)$, 
\[
\EE[{h_{v,0}^2(Y)}]\leq \frac{\EE[Y^{2v}]}{n^{2v}}.
\]
Furthermore, if for some positive constant $c'$, both $np$ and $2v$ are at most $c' \log n$, 
and $v\ge 1$,
\[
\EE[{h_{v,0}^2(Y)}]\leq 2p\Paren{\frac{2c'\log n}{n}}^{2v-1}.
\]
\end{lemma}
\begin{proof}
We consider the first inequality. Note that for all $y\in \mathbb{Z}^+$, 
\[
 0 \leq \prod_{l'=0}^{v-1}\Paren{y-l'}= \indic_{y\geq v} \cdot \prod_{l'=0}^{v-1}\Paren{y-l'}\leq y^v.
\]
This inequality trivially implies that
\[
\EE[{h_{v,0}^2(Y)}]=\frac{1}{n^{2v}} \EE\Paren{\prod_{l'=0}^{v-1}\Paren{Y-l'}}^2\leq \frac{\EE[Y^{2v}]}{n^{2v}}.
\]
Based on the first inequality, we prove the second one as follows.
\begin{align*}
\EE[{h_{v,0}^2(Y)}]
&
\leq \frac{\EE[Y^{2v}]}{n^{2v}}\\
&
\leq\frac{1}{n^{2v}}\sum_{t=1}^{2v} t^{2v-t}\binom{2v}{t}(np)^t\\
&
\leq\frac{1}{n^{2v}}\sum_{t=1}^{2v} {(2v)}^{2v-t}\binom{2v}{t}(c'\log n)^t\frac{np}{c'\log n}\\
&
\leq \frac{1}{n^{2v}}(2v+c'\log n)^{2v}\frac{np}{c'\log n}\\
&
\leq 2p\Paren{\frac{2c'\log n}{n}}^{2v-1}.\qedhere
\end{align*}
\end{proof}
\begin{lemma}\cite{jiaol1}\label{boundg2}
For a Poisson random variable $Y\sim \Poi(np)$ and a parameter 
\[
M\geq\max\Brace{\frac{n(p-x)^2}{p}, v},
\]
we have
\[
\EE[{h_{v,0}^2(Y)}]\leq \Paren{\frac{2Mp}{n}}^v.
\]
\end{lemma}

\begin{lemma}\cite{JiaoLLM}\label{boundg3}
For $x\in[0,1]$, $v\in \mathbb{N}$, $m\in \mathbb{N}$, and a parameter 
\[
 \delta\ge \max\Brace{\Abs{x-\frac mn},\frac{\sqrt{4mv}}{n}},
\]
we have
\[
\Abs{h_{v,x}({m})}\leq (2\delta)^v.
\]
\end{lemma}
\subsection{Function estimator construction}\label{sec:estconst}
Let $g$ be a continuous real function over the unit interval. Given \iid\ samples $X^n$ from a Bernoulli distribution
 with unknown success probability $p$, our objective is to estimate the function
value $g(p)$. 
 
\paragraph{Poisson sampling and sample splitting}
Generating exactly $n$ samples creates dependencies between the counts of symbols. To simplify the derivations, we use the well-known \emph{Poisson sampling}  technique and make the sample size an independent Poisson variable $N$ with mean $n$. 
In addition, we apply the standard \emph{sample splitting} method and
divide the sample sequence $X^N$ into two sub-sample sequences by independently putting each sample into one of the two with equal probability. 
Equivalently, we can simply generate two independent sample sequences from Bern$(p)$, each of an independent Poi$(n/2)$ size. 
For notational convenience, we replace $n$ by $2n$ and denote by $N_1$ and $N_1'$ the number of times symbol $1$ appearing in the first and second sample sequences, respectively.

\paragraph{Covering the unit interval}
Let $c$ be a sufficiently large constant and define $c_n := c\frac{\log n}n$. 
Cover the unit interval $[0,1]$ by three sets of nested intervals 
\begin{align*}
I_j&:=c_n \left[(j-1)^2,j^2\right],\\
 I^{*}_j&:=c_n \left[(j-2)^2\indic_{j\geq 2},(j+1)^2\right] = \union\limits_{j'=j-1}^{j+1} I_{j'},\\
I^{**}_j&:=c_n \left[(j-3)^2\indic_{j\geq 3},(j+2)^2\right] =
\union\limits_{j'=j-1}^{j+1} I_{j'}^*,
\end{align*}
where $j=1,\ldots\ $ and in the union, $I_{-2}$ and $I_{-1}$ are taken to be empty. 

Let $M_{n}:=1/\sqrt{c_n }$ be the number of intervals so that $I_1,\ldots, I_{M_{n}}$ form a partition of $[0,1]$. 

Parameter $c$ and these intervals are chosen so that for all $j\in[M_{n}]$, if ${N_1}/n\in I_j$
we can assume that $p\in I_j^*$ and ${N'_1}/n\in I^{**}_j$,
and regardless of the value of
$p$, with high probability we will be right. 

\paragraph{Min-max polynomial approximation}
For each $j\in[M_{n}]$, let 
$x_j:=c_n (j-3)^2\indic_{j\geq 3}$
be the left end point of $I^{**}_j$, and
$
|I^{**}_j|:=c_n (j+2)^2-c_n (j-3)^2\indic_{j\geq 3}
$
be the length of the interval $I^{**}_j$.

Then for any $x\in  I^{**}_j$, there exists $y_x\in[0,1]$ such that $x=x_j+|I^{**}_j|\cdot y_x$.
Let $\lambda$ be a small absolute constant in $(0,0.1)$, and define the \emph{degree parameter} as
\[
\degree:=\max\Brace{d\in{\mathbb{N}}:\ d\cdot2^{4.5d+2}\leq {n}^{\lambda}}.
\]
Denoting
\[
r_j(y):=g\Paren{x_j+|I^{**}_j|y},
\]
we can find the degree-$\degree$ min-max polynomial of $r_j(y)$ over $y\in[0,1]$, say 
\[
\tilde{r}_j(y):=\sum_{v=0}^{\degree}a_{jv} y^v.
\]
By Lemma~\ref{coeffbound}, for all $v\geq{1}$, the following upper bound on $|a_{jv}|$ holds.
\[
|a_{jv}|\leq 2^{3.5\degree} \sup_{z_1,z_2\in  I^{**}_j}|g(z_1)-g(z_2)|.
\]
Noting that $y_x=|I^{**}_j|^{-1}(x-x_j)$, we can re-write $\tilde{r}_j(y_x)$ as
\[
\tilde{g}_j(x):=\sum_{v=0}^{\degree}a_{jv}  |I^{**}_j|^{-v} (x-x_j)^v .
\]
In addition, we vertically shift the polynomial slightly so that $\tilde g_j(x)$
coincides with $g(x)$ at $x_j$. 
Note that this is equivalent to setting $a_{jv}=g(x_j)$.

\paragraph{Piecewise-polynomial estimator $\boldsymbol{\hat{g}^*}$}
By Lemma~\ref{boundg0}, for $j\in [M_n]$, an unbiased estimator of~$\tilde{g}_j(p)$~is
\[
E_{\tilde{g}_j}(N_1)
:=\sum_{v=0}^{\degree}a_{jv} |I^{**}_j|^{-v}h_{v,x_j}\Paren{N_1}
=\sum_{v=0}^{\degree}a_{jv} |I^{**}_j|^{-v} \sum_{l=0}^{v}\binom{v}{l}(-x_j)^{v-l}\prod_{l'=0}^{l-1}\Paren{\frac{N_1}{n}-\frac{l'}{n}}.
\]
For $j>M_{n}$, we denote
\[
E_{\tilde{g}_j}(N_1):=E_{\tilde{g}_{M_{n}}}(\min\Brace{N_1, c_n (M_{n}+2)^2}).
\]

Let $T$ be a sufficiently large constant satisfying $T\gg \max_{x\in [0,1]}|g(x)|$, and write $[A]_a^b$ instead of $\min\{\max\{A, a\},b\}$. Utilizing sample splitting, we estimate $g(p)$ by the following estimator,
\[
\hat{g}^*(N_1, N_1'):=\Br{\sum_{j=1}^{\infty}\Paren{E_{\tilde{g}_j}(N_1)\indic_{\frac{N_1}{n}\in  I^{**}_{j}}+\sum_{j'\not\in[j-2:j+2]} E_{\tilde{g}_{j'}}(N_1)\indic_{\frac{N_1}{n}\in I_{j'}}}\indic_{\frac{N_1'}{n}\in I_j}}^{T}_{-T}.
\]

\subsection{Bounding the bias of $\boldsymbol{\hat{g}^*}$}\label{sec:bias}
Recall that $I_1,\ldots, I_{M_{n}}$ form a partition of $[0,1]$. For any $x\in[0,1]$, let $j_x$ denote the index $j$ such that $x\in I_j$.
By the triangle inequality, the absolute bias of $\hat{g}^*(N_1,N_1')$ admits 
\begin{align*}
|\EE[\hat{g}^*(N_1,N_1')]-g(p)|
&
\leq \left|\tilde{g}_{j_{p}}(p)-g(p)\right|+\left|\EE\left[\hat{g}^*(N_1,N_1')-\tilde{g}_{j_{p}}(p)\right]\right|\\
&\leq 
\left|\tilde{g}_{j_{p}}(p)-g(p)\right|
+\EE\left[2T\Paren{\indic_{\frac{N_1}{n}\not\in  I^{*}_{j_{p}}}\indic_{\frac{N_1'}{n}\in  I^{*}_{j_{p}}}+\indic_{\frac{N_1'}{n}\not\in  I^{*}_{j_{p}}}}\right]\\
&
+\left|\EE\left[\Paren{\hat{g}^*(N_1,N_1')-\tilde{g}_{j_{p}}(p)}\indic_{\frac{N_1}{n}\in  I^{*}_{j_{p}}}\indic_{\frac{N_1'}{n}\in  I^{*}_{j_{p}}}\right]\right|.
\end{align*}
The last summation has three terms. By definition, the first term is no larger than ${\Devstar_{g}(n,{p})}/{n}$.
By the Chernoff bound (Lemma~\ref{tailprob}) and the fact that $p\in I_{j_p}$, for sufficiently large constant $c$, the second term is at most $Tp/n^5$. 
Therefore, it remains to consider the third term. By the triangle inequality and definition of $\hat{g}^*$, the third term is at most
\begin{align*}
B_n(g, p):=&\left|\EE\left[ \Paren{\tilde{g}_{j_{p}-1}(p)-\tilde{g}_{j_{p}}(p)}\indic_{\frac{N_1}{n}\in  I^{*}_{j_{p}}}\right]\right|
+\left|\EE\left[ \Paren{\tilde{g}_{j_{p}+1}(p)-\tilde{g}_{j_{p}}(p)}\indic_{\frac{N_1}{n}\in  I^{*}_{j_{p}}}\right]\right|\\
&+\left|\EE\left[ \Paren{E_{\tilde{g}_{j_{p}}}(N_1)-\tilde{g}_{j_{p}}(p)}\indic_{\frac{N_1}{n}\in  I^{*}_{j_{p}}}\right]\right|+\left|\EE\left[ \Paren{E_{\tilde{g}_{j_{p}-1}}(N_1)-\tilde{g}_{j_{p}-1}(p)}\indic_{\frac{N_1}{n}\in  I^{*}_{j_{p}}}\right]\right|\\
&+\left|\EE\left[ \Paren{E_{\tilde{g}_{j_{p}+1}}(N_1)-\tilde{g}_{j_{p}+1}(p)}\indic_{\frac{N_1}{n}\in  I^{*}_{j_{p}}}\right]\right|.
\end{align*}

We bound the first term of $B_n(g, p)$ as
\begin{align*}
\left|\EE\left[ \Paren{\tilde{g}_{j_{p}-1}(p)-\tilde{g}_{j_{p}}(p)}\indic_{\frac{N_1}{n}\in  I^{*}_{j_{p}}}\right]\right|
&\leq \left|\EE\left[\tilde{g}_{j_{p}-1}(p)-\tilde{g}_{j_{p}}(p)\right]\right|\\
&\leq \left|\tilde{g}_{j_{p}-1}(p)-g(p)\right|+\left|g(p)-\tilde{g}_{j_{p}}(p)\right|\\
&\leq \frac{2\Devstar_{g}(2n,{p})}{n},
\end{align*}
where the last step follows from the definition of $\Devstar_{g}(2n,{p})$. The second term of $B_n(g, p)$ satisfies the same inequality, and is at most ${2\Devstar_{g}(2n,{p})}/{n}$. 
Note that the last three terms of $B_n(g, p)$ are clearly of the same type. Hence for simplicity, 
below we analyze only the first one.

For any $j\in[M_{n}]$, we can express $E_{\tilde{g}_j}(N_1)$ in terms of $h_{v,x_j}({N_1})$, i.e.,
\[
E_{\tilde{g}_j}(N_1)=\sum_{v=0}^{\degree}a_{jv} |I^{**}_j|^{-v} h_{v,x_j}\Paren{{N_1}}.
\]
In addition, recall that by definition, 
\[
\tilde{g}_j(p)=\sum_{v=0}^{\degree}a_{jv} |I^{**}_j|^{-v} (p-x_j)^v.
\]
The linearity of expectation combines the above two equalities and yields
\[
\EE\left[ \Paren{E_{\tilde{g}_j}(N_1)-\tilde{g}_j(p)}\indic_{\frac{N_1}{n}\in  I^{*}_j}\right]
=\sum_{v=1}^{\degree}a_{jv} |I^{**}_j|^{-v} \EE \Br{\Paren{h_{v,x_j}(N_1)- (p-x_j)^v}\indic_{\frac{N_1}{n}\in  I^{*}_j}},
\]
where the constant term cancels out. 
Therefore, given integers $a$ and $b$ satisfying $b>a>\degree$, our \emph{new objective} is to bound
\[
\textit{IN}_{v,n}(a, b, p, j):=\EE \Br{\Paren{h_{v,x_j}(N_1)- (p-x_j)^v}\indic_{{N_1}\in [a,b]}}.
\]

\paragraph{Bounding the magnitude of $\boldsymbol{\textit{IN}_{v,n}}$}
For all integer $s\geq 1$, let us denote
\[
H_{v,n}(s,p,j):=\sum_{l=0}^{v}\binom{v}{l}(-x_j)^{v-l}p^l\sum_{t\in[s-l,s-1]} e^{-np}\frac{(np)^t}{t!}.
\]
We first relate $\textit{IN}_{v,n}$ to $H_{v,n}$ through the following lemma.

\begin{lemma}\label{INtoH}
For any two integers $a$ and $b$ satisfying $a>b>v$,
\[
\textit{IN}_{v,n}(a, b, p, j)=H_{v,n}(a,p,j)-H_{v,n}(b+1,p,j).
\]
\end{lemma}
\begin{proof}
By the linearity of expectation and binomial theorem, we can rewrite the left-hand side as
\[
\textit{IN}_{v,n}(a, b, p, j)= \sum_{l=0}^{v}\binom{v}{l}(-x_j)^{v-l}\EE\Br{\Paren{\prod_{l'=0}^{l-1}\Paren{\frac{N_1}{n}-\frac{l'}{n}}-p^l}\indic_{{N_1}\in [a,b]}}.
\]
For each $l\leq v$, we evaluate the inner expectation as follows:
\begin{align*}
\EE\Br{\Paren{\prod_{l'=0}^{l-1}\Paren{\frac{N_1}{n}-\frac{l'}{n}}-p^l}\indic_{{N_1}\in [a,b]}}
&=\sum_{t\in[a,b]} \prod_{l'=0}^{l-1}\Paren{\frac{t}{n}-\frac{l'}{n}} e^{-np}\frac{(np)^t}{t!}-p^l\sum_{t\in[a,b]} e^{-np}\frac{(np)^t}{t!}\\
&=\sum_{t\in[a,b]} \frac{1}{n^l}\frac{t!}{(t-l)!} e^{-np}\frac{(np)^t}{t!}-p_i^l\sum_{t\in[a,b]} e^{-np}\frac{(np)^t}{t!}\\
&=p^l\sum_{t\in[a-l,b-l]} e^{-np}\frac{(np)^t}{t!}-p^l\sum_{t\in[a,b]} e^{-np}\frac{(np)^t}{t!}\\
&=p^l\sum_{t\in[a-l,a-1]} e^{-np}\frac{(np)^t}{t!}-p^l\sum_{t\in[b-l+1,b]} e^{-np}\frac{(np)^t}{t!}.\qedhere
\end{align*}
\end{proof}
Therefore, to bound $|\textit{IN}_{v,n}(a, b, p, j)|$, 
we need to bound only $|H_{v,n}(a,p,j)|$ and $|H_{v,n}(b+1,p,j)|$.
We shall proceed by relating these quantities to $h_{l,x_j}({a-1})$ for $l=0,\ldots,v-1$.

\begin{lemma}\label{expH}
For any integer $s$ satisfying $s>v$,
\[
H_{v,n}(s,p,j) = pe^{-np}\frac{(np)^{s-1}}{(s-1)!}\sum_{l=0}^{v-1} (p-x_j)^{v-l-1} h_{l,x_j}(s-1).
\]
\end{lemma}
\begin{proof}
The following recursive formula of binomial coefficients is well-known:
\[
\binom{v}{l}=\binom{v-1}{l}+\binom{v-1}{l-1},
\]
Utilizing this recursive formula, we can re-write the quantity of interest as
\begin{align*}
H_{v,n}(s,p,j)
&=\sum_{l=0}^{v-1}{\binom{v-1}{l}}(-x_j)^{v-{(l+1)}}p^{l+1}\sum_{t\in[s-l,s-1]} e^{-np}\frac{(np)^t}{t!}\\
&+\sum_{l=0}^{v-1}{\binom{v-1}{l}}(-x_j)^{v-l}p^l\sum_{t\in[s-l,s-1]} e^{-np}\frac{(np)^t}{t!}\\
&+\sum_{l=0}^{v-1}{\binom{v-1}{l}}(-x_j)^{v-{(l+1)}}p^{l+1}e^{-np}\frac{(np)^t}{t!}\Bigg|_{t=s-(l+1)}\\
&=(p-x_j)\sum_{l=0}^{v-1}{\binom{v-1}{l}}(-x_j)^{(v-1)-l}p^{l}\sum_{t\in[s-l,s-1]} e^{-np}\frac{(np)^t}{t!}\\
&+\sum_{l=0}^{v-1}{\binom{v-1}{l}}(-x_j)^{v-{(l+1)}}p^{l+1}e^{-np}\frac{(np)^{s-(l+1)}}{(s-(l+1))!}\\
&=(p-x_j)H_{v-1,n}(s,p,j)+\sum_{l=0}^{v-1}{\binom{v-1}{l}}(-x_j)^{v-{(l+1)}}p^{l+1}e^{-np}\frac{(np)^{s-(l+1)}}{(s-(l+1))!}.
\end{align*}
This equation establishes a standard recursive relation between $H_{v,n}(s,p,j)$ and $H_{v-1,n}(s,p,j)$. To~prove our desired result, we relate the second quantity on the right-hand side to $h_{v-1,x_j}(s-1)$,
\begin{align*}
&\sum_{l=0}^{v-1}{\binom{v-1}{l}}(-x_j)^{v-{(l+1)}}p^{l+1}e^{-np}\frac{(np)^{s-(l+1)}}{(s-(l+1))!}\\
&=e^{-np}p\sum_{l=0}^{v-1}{\binom{v-1}{l}}(-x_j)^{(v-1)-l}p_x^{l}\frac{(np)^{(s-1)-l}}{((s-1)-l)!}\\
&=e^{-np}p\sum_{l=0}^{v-1}{\binom{v-1}{l}}(-x_j)^{(v-1)-l}\frac{1}{n^l}\frac{(np)^{s-1}}{((s-1)-l)!}\\
&=e^{-np}p\sum_{l=0}^{v-1}{\binom{v-1}{l}}(-x_j)^{(v-1)-l}\frac{\prod_{l'=0}^{l-1} ((s-1)-l')}{n^l}\frac{(np)^{s-1}}{(s-1)!}\\
&=e^{-np}p\frac{(np)^{s-1}}{(s-1)!}\sum_{l=0}^{v-1}{\binom{v-1}{l}}(-x_j)^{(v-1)-l}\prod_{l'=0}^{l-1} \Paren{\frac{s-1}{n}-\frac{l'}{n}}\\
&=pe^{-np}\frac{(np)^{s-1}}{(s-1)!}h_{v-1,x_j}(s-1).
\end{align*}
Substituting the last quantity into the previous recursive relation yields
\[
H_{v,n}(s,p,j)=(p-x_j)H_{v-1,n}(s,p,j)+pe^{-np}\frac{(np)^{s-1}}{(s-1)!}h_{v-1,x_j}(s-1),
\]
with a base case $H_{0,n}(s,p,j)=0$. Therefore, the principle of mathematical induction implies
\[
H_{v,n}(s,p,j) = pe^{-np}\frac{(np)^{s-1}}{(s-1)!}\sum_{l=0}^{v-1} (p-x_j)^{v-l-1} h_{l,x_j}(s-1).\qedhere
\]
\end{proof}
Without loss of generality, we assume that $c\log n$ is a positive integer so that $nx_j\in \mathbb{Z}^+$ for~all~$j$, since otherwise we can modify the value of $c$ by at most $1$ to fulfill this assumption.

As an implication of Lemma~\ref{expH}, for integer $s$ such that $s/n$ or $(s-1)/n$ is the end point of  $ I^{*}_{j_p}$ (right end point if ${j_p}\leq 2$), and sufficiently large constant $c$ satisfying $c\log n>d$,  
\begin{align*}
\Abs{H_{v,n}(s,p,{j_p})}
&=\Abs{pe^{-np}\frac{(np)^{s-1}}{(s-1)!}\sum_{l=0}^{v-1} (p-x_{j_p})^{v-l-1} h_{l,x_{j_p}}(s-1)}\\
&=\Pr(N_1=s-1)\cdot p\Abs{\sum_{l=0}^{v-1} (p-x_{j_p})^{v-l-1} h_{l,x_{j_p}}(s-1)}\\
&\leq \frac{p}{n^5}\Abs{\sum_{l=0}^{v-1} (p-x_{j_p})^{v-l-1} h_{l,x_{j_p}}(s-1)}\\
&\leq \frac{p}{n^5}v \Paren{2 |I^{**}_{j_p}|}^{v-1},
\end{align*}
where the second last step follows from the Chernoff bound and the last step follows from Lemma~\ref{boundg3} by setting $\delta= |I^{**}_{j_p}|$.
Under the same set of conditions, we can show that 
\[
\Abs{H_{v,n}(s,p,j_p-1)}\leq \frac{p}{n^5}v \Paren{2 |I^{**}_{j_p}|}^{v-1}
\]
and 
\[
\Abs{H_{v,n}(s,p,j_p+1)}\leq \frac{p}{n^5}v \Paren{2 |I^{**}_{j_p}|}^{v-1}.
\]
\paragraph{Bounding the bias of $\boldsymbol{\hat{g}^*}$} Now we are ready to analyze the quantity of interest:
\begin{align*}
\Abs{\EE\left[ \Paren{E_{\tilde{g}_{j_p}}(N_1)-\tilde{g}_{j_p}(p)}\indic_{\frac{N_1}{n}\in  I^{*}_{j_p}}\right]}
&=\Abs{\sum_{v=1}^{\degree}a_{j_pv} |I^{**}_{j_p}|^{-v} \EE \Br{\Paren{h_{v,x_{j_p}}(N_1)- (p-x_{j_p})^v}\indic_{\frac{N_1}{n}\in  I^{*}_{j_p}}}}\\
&=\Abs{\sum_{v=1}^{\degree}a_{j_pv} |I^{**}_{j_p}|^{-v} \textit{IN}_{v,n}(nx_{j_p+1}, nx_{j_p+4}, p, j_p)}\\
&=\Abs{\sum_{v=1}^{\degree}a_{j_pv} |I^{**}_{j_p}|^{-v} (H_{v,n}(nx_{j_p+1},p,j)\indic_{j_p>2}-H_{v,n}(nx_{j_p+4}+1,p,j))}\\
&\leq \sum_{v=1}^{\degree}a_{j_pv} |I^{**}_{j_p}|^{-v}  \frac{2p}{n^5}v \Paren{2 |I^{**}_{j_p}|}^{v-1}\\
&\leq \sum_{v=1}^{\degree}2T\cdot2^{3.5\degree}\Paren{\frac{1}{4c_n }}  \frac{2p}{n^5}v \Paren{2}^{v-1}\\
&\leq   \frac{T\degree\cdot 2^{4.5\degree}}{c_n n^5}\cdot p.
\end{align*}

The same reasoning also shows that
\[
\left|\EE\left[ \Paren{E_{\tilde{g}_{j_p-1}}(N_1)-\tilde{g}_{j_p-1}(p)}\indic_{\frac{N_1}{n}\in  I^{*}_{j_p}}\right]\right|\leq \frac{T\degree\cdot 2^{4.5\degree}}{c_n n^5}\cdot p
\]
and 
\[
\left|\EE\left[ \Paren{E_{\tilde{g}_{j_p+1}}(N_1)-\tilde{g}_{j_p+1}(p)}\indic_{\frac{N_1}{n}\in  I^{*}_{j_p}}\right]\right|\leq \frac{T\degree\cdot 2^{4.5\degree}}{c_n n^5}\cdot p.
\]
Consolidating all the previous results yields the desired bias bound:
\begin{align*}
|\EE[\hat{g}^*(N_1,N_1')]-g(p)|
&\leq \frac{T}{n^5}\cdot p+\frac{3T\degree\cdot 2^{4.5\degree}}{c_n n^5}\cdot p+\frac{5}{n}\Devstar_{g}(2n,{p})\\
&\leq \frac{p}{n^{5-\lambda}}+\frac{5}{n}\Devstar_{g}(2n,{p}).
\end{align*}

\subsection{Bounding the variance of $\boldsymbol{\hat{g}^*}$}\label{sec:variance}\vspace{-0.5em}
In this section, we establish the following bound on the variance of our estimator.
\begin{lemma}\label{lem:variance}
For sufficiently large $c$,
\[
\Var(\hat{g}^*(N_1, N_1'))\leq  \frac{2c(\log n)}{n^{1-3\lambda}}\Paren{36 \Lipstar_{g}(2n, {p})}^2\cdot p+\frac{6T^2}{n^5}\cdot p.
\]\par\vspace{-1em}
\end{lemma}
\begin{proof}
\newcommand{\overbar}[1]{\mkern 1.5mu\overline{\mkern-1.5mu#1\mkern-1.5mu}\mkern 1.5mu}
Since $\Var{(X)}\leq \EE[X^2]$ and $\indic_{X}\cdot \indic_{\overbar{X}}=0$ for any random variable $X$, we have
\begin{align*}
\Var(\hat{g}^*(N_1,N_1'))
&
\leq 2\Var\Paren{(\hat{g}^*(N_1,N_1')-g(x_{j_p-1}))\indic_{\frac{N_1}{n}\in  I^{*}_{j_{p}}}\indic_{\frac{N_1'}{n}\in  I^{*}_{j_{p}}}+g(x_{j_p-1})\indic_{\frac{N_1}{n}\in  I^{*}_{j_{p}}}\indic_{\frac{N_1'}{n}\in  I^{*}_{j_{p}}}}\\
&
+2\Var\Paren{\hat{g}^*(N_1,N_1')\Paren{1-\indic_{\frac{N_1}{n}\in  I^{*}_{j_{p}}}\indic_{\frac{N_1'}{n}\in  I^{*}_{j_{p}}}}}\\
&
\leq 
12\sum_{j\in[j_{p}-1,j_{p}+1]}\Var\Paren{(\hat{g}^*(N_1,N_1')-g(x_{j_p-1}))\indic_{\frac{N_1}{n}\in  I^{*}_{j_{p}}}\indic_{\frac{N_1'}{n}\in  I_{j}}}\\
&
+6T^2\cdot\Pr{\Paren{\frac{N_1}{n}\not\in  I^{*}_{j_{p}}\text{ or } \frac{N_1'}{n}\not\in  I^{*}_{j_{p}}}}.\\
&
\leq 
12\sum_{j\in[j_{p}-1,j_{p}+1]} \EE[(E_{\tilde{g}_{j}}(N_1)-g(x_{j_p-1}))^2]+6T^2\cdot\Pr{\Paren{\frac{N_1}{n}\not\in  I^{*}_{j_{p}}}}.
\end{align*}\par\vspace{-0.5em}

For sufficiently large $c$, the second term is at most $8T^2p/n^5$ by the Chernoff bound. It remains to analyze $\EE[(E_{\tilde{g}_{j}}(N_1)-g(x_{j_p-1}))^2]$ for $j\in[j_{p}-1,j_{p}+1]$. By the Cauchy-Schwarz inequality, 
\begin{align*}
\EE[(E_{\tilde{g}_{j}}(N_1)-g(x_{j_p-1}))^2]
&=\EE\Paren{\sum_{v=1}^{\degree}a_{jv} |I^{**}_j|^{-v} h_{v,x_j}(N_1)+(g(x_j)-g(x_{j_p-1}))}^2\\
&\leq\Paren{\sum_{v=1}^{\degree}|a_{jv}| |I^{**}_j|^{-v} \Paren{\EE[h_{v,x_j}^2(N_1)]}^{\frac12}
+|g(x_j)-g(x_{j_p-1})|}^2.
\end{align*}
Consider the inner expectation.
If $j_{p_i}\leq 2$ and $j\in[j_{p_i}-1,j_{p_i}+1]$, then $x_j=x_{j_p-1}=0$. 
Hence the constant term $|g(x_j)-g(x_{j_p-1})|$ equals to zero. 
By Lemma~\ref{boundg1}, for $v\ge 1$,
\begin{align*}
\EE [h_{v,x_j}^2(N_1)]
&\leq \frac{2(32c\log n)^{2v-1}p}{n^{2v-1}}.
\end{align*}
This together with Lemma~\ref{coeffbound} and the definition of $\Lipstar_{g}(n,{p})$ implies that
\begin{align*}
\EE[(E_{\tilde{g}_{j}}(N_1)-g(x_{j_p-1}))^2]
&\leq\Paren{\sum_{v=1}^{\degree}|a_{jv}| |I^{**}_j|^{-v} \Paren{\EE[h_{v,x_j}^2(N_1)]}^{\frac12}}^2\\
&\leq\Paren{\sum_{v=1}^{\degree}\Paren{2^{3.5\degree+1} \Lipstar_{g}(2n,{p})|I^{**}_j|} |I^{**}_j|^{-v} \Paren{ \frac{32c\log n}{n}}^{v-1}\sqrt{\frac{64c(\log n)p}{n}}}^2\\
&\leq \Paren{d_n2^{5.5\degree+1} \Lipstar_{g}(2n,{p})}^2 \frac{64c(\log n)p}{n}.
\end{align*}
If $j_{p}>2$ and $j\in[j_{p}-1,j_{p}+1]$, then by Lemma~\ref{boundg2},
\begin{align*}
\EE [h_{v,x_j}^2(N_1)]
\leq \Paren{\frac{72c(\log n) p}{n}}^v\leq \Paren{\frac{72c^2(\log n)^2j_{p}^2}{n^2}}^{v-1}\Paren{\frac{72c(\log n) p}{n}}.
\end{align*}
As for the constant term, direct computation shows that
\[
|g(x_j)-g(x_{j_p-1})|\le \Lipstar_{g}(2n,{p})\sqrt{\frac{36c(\log n)p}{n}}.
\]
Analogously,\vspace{-0.5em}
\begin{align*}
\EE[(E_{\tilde{g}_{j}}(N_1)-g(x_{j_p-1}))^2]
&\leq\Paren{\sum_{v=1}^{\degree}|a_{jv}| |I^{**}_j|^{-v} \Paren{\EE[h_{v,x_j}^2(N_1)]}^{\frac12}
+|g(x_j)-g(x_{j_p-1})|
}^2\\
&\leq\left(\sum_{v=1}^{\degree}\Paren{2^{3.5\degree+1}  \Lipstar_{g}(2n,{p}) |I^{**}_j|} |I^{**}_j|^{-v} 
\Paren{ \frac{\sqrt{72}c(\log n)j_{p}}{n}}^{v-1}
\right.\\
&\left.
\ \ \times \sqrt{\frac{72c(\log n)p}{n}}+ \Lipstar_{g}(2n,{p})\sqrt{\frac{36c(\log n)p}{n}} \right)^2\\
&\leq \Paren{d_n2^{4.5\degree+2}  \Lipstar_{g}(2n,{p})}^2 \frac{72c(\log n)p}{n}.
\end{align*}
Consolidating the above results yields the desired bound.
\end{proof}

\subsection{Sensitivity bound}\label{sec:conc}
Incorporate our sampling scheme, we define the \emph{sensitivity} of an estimator $\hat{g}$ as the maximum possible change in its value when an input sequence is replaced by another that differs in exactly one location,
\[
S(\hat{g})
:=
\max\left\{\Abs{\hat{g}(x^m)-\hat{g}(y^{m'})}: m, m'\in \mathbb{Z},\  x^m\text{ and
  }y^{m'} \text{ differ in one location}\right\}.
\]
By construction, sensitivity upperly bounds $n$-sensitivity, i.e., $S(\hat{g})\geq S(\hat{g},n)$ for all $n$. Due to sample splitting, 
replacing the given sample sequence $X^{N}$ by a sequence that differs in at most one location could change $N_1$, $N_1'$, or both, by at most one.
In other words, to bound the sensitivity of~$\hat{g}^*$, we need to bound the change in the estimator's value when we modify $N_1$ or $N_1'$ by one. We proceed as follows. 
If the value of $N_1$ increases or decreases by one, we need to consider the following two types of differences:
\[
\mathbb{D}_{g}^{(1)}(n, j, s) := E_{\tilde{g}_j}(s)-E_{\tilde{g}_j}(s-1),
\]
for $s$ satisfying $s-1, s, \text{ or } s+1\in   nI^{**}_{j}$, and 
\[
\mathbb{D}_{g}^{(3)}(n, j, s) := E_{\tilde{g}_j}(s)-E_{\tilde{g}_{j-1}}(s-1),
\]
for $s$ satisfying $s\in nI^{**}_{j-1}\cap nI^{**}_{j}$.
If the value of $N_1'$ increases or decreases by one, we need to consider the difference:
\[
\mathbb{D}_{g}^{(2)}(n, j, s) := E_{\tilde{g}_j}(s)-E_{\tilde{g}_{j-1}}(s),
\]
for $s$ satisfying $s\in nI^{**}_{j-1}\cap nI^{**}_{j}$.
The triangle inequality relates these three quantities and yields
\begin{align*}
\Abs{\mathbb{D}_{g}^{(3)}(n, j, s)}
&=\Abs{E_{\tilde{g}_j}(s)-E_{\tilde{g}_{j-1}}(s-1)}\\
&\leq \Abs{E_{\tilde{g}_j}(s)-E_{\tilde{g}_j}(s-1)}+\Abs{E_{\tilde{g}_j}(s-1)-E_{\tilde{g}_{j-1}}(s-1)}\\
&= \Abs{\mathbb{D}_{g}^{(1)}(n, j, s)}+\Abs{\mathbb{D}_{g}^{(2)}(n, j, s-1)}.
\end{align*}
Hence to bound $S(\hat{g})$, we need to derive upper bounds for only 
$|\mathbb{D}_{g}^{(1)}|$ and $|\mathbb{D}_{g}^{(2)}|$, which we refer to as the \emph{type-1} and \emph{type-2 differences}, respectively. In Section~\ref{sec:type1} and~\ref{sec:type2}, we show that both quantities are at most ${\Sensitivity_g(2n)}/{n^{1-\lambda}}$. Given this, and a Poisson-sampling McDiarmid's inequality derived in the next section, we establish the third inequality in Theorem 4. 

\subsubsection{From bounded difference to concentration}\label{conc}
In this section, we establish a McDiarmid's inequality for Poisson sampling, showing that small sensitivity still implies strong concentration under formulation. We believe that this result is of independent interest.
Specifically, we show that for any $p\in \Delta_k$, $N\sim \Poi(n)$, $X^N\sim p$,  
\begin{lemma}\label{trueMcD}
For any error parameter $\varepsilon\in(0,1)$ and estimator $\hat{f}$ satisfying $S(\hat{f})\ge 1/n$,
\begin{align*}
\Pr\Paren{\Abs{\hat{f}(X^{N})-\EE\Br{\hat{f}(X^{N})}}>\varepsilon}
\leq4\exp\Paren{-\frac{\varepsilon^2}{2n(4S(\hat{f}))^2}}.
\end{align*}
\end{lemma}
\begin{proof}
By the linearity of expectation and triangle inequality, 
\[
|\EE[\hat{f}(X^{m})]-\EE[\hat{f}(X^{m+1})]|\leq S(\hat{f}), \forall m.
\]
Therefore for any $m$, 
\begin{align*}
|\EE[\hat{f}(X^{m})]-\EE[\hat{f}(X^{N})]|
&
=\Abs{\sum_{t=0}^{\infty}\EE[\hat{f}(X^{t})]\cdot\Pr(N=t)-\EE[\hat{f}(X^{m})]}\\
&
=\Abs{\sum_{t=0}^{\infty}(\EE[\hat{f}(X^{t})]-\EE[\hat{f}(X^{m})])\cdot\Pr(N=t)}\\
&
\leq S(\hat{f}) \sum_{t=0}^{\infty}|t-m|\cdot\Pr(N=t).
\end{align*}
We consider the last summation and simplify it as follows:
\begin{align*}
&
\sum_{t=0}^{\infty}|t-m|\cdot \Pr(N=t)\\
&
=\sum_{t=0}^{m}(m-t)\Pr(N=t)+\sum_{t=m}^{\infty}(t-m)\Pr(N=t)\\
&
=m\Pr(N\leq m)-\sum_{t=0}^{m} t\exp(-n)\frac{n^t}{t!}+\sum_{t=m}^{\infty} t\exp(-n)\frac{n^t}{t!}-m\Pr(N\geq m)\\
&
=m\Pr(N\leq m)-n\Pr(N\leq m-1)+n\Pr(N\geq m-1)-m\Pr(N\geq m)\\
&
=(m-n)(\Pr(N\leq m)-\Pr(N\geq m))+n(\Pr(N= m)+\Pr(N= m-1)).
\end{align*}
Note that the second quantity on the right-hand side satisfies
\begin{align*}
\Pr(N= m)+\Pr(N= m-1)
&\leq \Pr(N= n)+\Pr(N= n-1)\\
&\leq 2\exp(-n)\frac{n^{n}}{n!}\\
&\leq \frac{1}{\sqrt{n}}.
\end{align*}
Consequently we have
\begin{align*}
|\EE[\hat{f}(X^{m})]-\EE[\hat{f}(X^{N})]|
&
\leq S(\hat{f})\sum_{t=0}^{\infty}|t-m|\cdot \Pr(N=t)\\
&
\leq S(\hat{f})\Paren{(m-n)(\Pr(N\leq m)-\Pr(N\geq m))+\sqrt{n}}\\
&
\leq S(\hat{f})\cdot (|m-n|+\sqrt{n}).
\end{align*}
Next, let $\varepsilon'\in(0,1)$ be a constant to be determined later. The probability of interest satisfies
\begin{align*}
&\Pr\Paren{\Abs{\hat{f}(X^{N})-\EE\Br{\hat{f}(X^{N})}}>\varepsilon}\\
&
=\sum_{m=0}^{\infty}\Pr\Paren{\Abs{{\hat{f}(X^{m})-\EE[\hat{f}(X^{N})]}}>\varepsilon}\Pr(N=m)\\
&
\leq \Pr(N\not\in n[1-\varepsilon',1+\varepsilon'])+\sum_{m\in n[1-\varepsilon',1+\varepsilon']}\Pr\Paren{\Abs{{\hat{f}(X^{m})-\EE[\hat{f}(X^{N})]}}>\varepsilon}\Pr(N=m).
\end{align*}
We can easily bound the first term through the Chernoff bound. For the second term,
\begin{align*}
&
\sum_{m\in n[1-\varepsilon',1+\varepsilon']}\Pr\Paren{\Abs{{\hat{f}(X^{m})-\EE[\hat{f}(X^{m})]}}>\varepsilon-\Abs{\EE[\hat{f}(X^{m})]-\EE[\hat{f}(X^{N})]}}\Pr(N=m)\\
&
\leq\sum_{m\in n[1-\varepsilon',1+\varepsilon']}\Pr\Paren{\Abs{{\hat{f}^*(X^{m})-\EE[\hat{f}^*(X^{m})]}}>\varepsilon-S(\hat{f})(n\varepsilon'+\sqrt{n})}\Pr(N=m)\\
&
\leq 2\exp\Paren{-\frac{(\varepsilon-S(\hat{f})(n\varepsilon'+\sqrt{n}))^2}{n(1+\varepsilon')(S(\hat{f}))^2}},
\end{align*}
where the last step follows from McDiarmid's inequality. Next, setting 
\[
\varepsilon'=\frac{\varepsilon}{2nS(\hat{f})}\in\Paren{0, \frac12},
\]
we can rewrite last term, with the multiplicative factor of 2 removed, as
\[
\exp\Paren{-\frac{(\varepsilon-S(\hat{f})(n\varepsilon'+\sqrt{n}))^2}{n(1+\varepsilon')(S(\hat{f}))^2}}=\exp\Paren{-\frac{(\frac{\varepsilon}{2}-\sqrt{n}S(\hat{f}))^2}{n(1+\varepsilon')(S(\hat{f}))^2}}.
\]
Hence, it suffices to obtain tight upper bounds on the right-hand side quantity, for which we consider the following two cases.
If the parameter $\varepsilon$ is relatively large such that 
\[
\varepsilon\ge 4\sqrt{n}S(\hat{f}),
\]
 the quantity of interest is at most 
\[
\exp\Paren{-\frac{(\frac{\varepsilon}{2}-\sqrt{n}S(\hat{f}))^2}{n(1+\varepsilon')(S(\hat{f}))^2}}\leq \exp\Paren{-\frac{\varepsilon^2}{32n(S(\hat{f}))^2}}.
\]
Otherwise, we have
$
{\varepsilon^2}/{(32(S(\hat{f}))^2)}\leq {1}/{2}, 
$
implying 
\[
2\exp\Paren{-\frac{\varepsilon^2}{32n(S(\hat{f}))^2}}\ge 2\exp\Paren{-\frac12}>1.
\]
Consolidating previous results, we get
\begin{align*}
&
\Pr\Paren{\Abs{\hat{f}^*(X^{N''})-\EE\Br{\hat{f}^*(X^{N''})}}>\varepsilon}\\
&
\leq2\exp\Paren{-\frac{\varepsilon^2}{32n(S(\hat{f}))^2}}+\Pr(N''\not\in n[1-\varepsilon',1+\varepsilon'])\\
&
\leq2\exp\Paren{-\frac{\varepsilon^2}{32n(S(\hat{f}))^2}}+2\exp\Paren{-\frac{1}{3}n\varepsilon'^2}\\
&
\leq2\exp\Paren{-\frac{\varepsilon^2}{32n(S(\hat{f}))^2}}+2\exp\Paren{-\frac{\varepsilon^2}{12n(S(\hat{f}))^2}}\\
&
\leq4\exp\Paren{-\frac{\varepsilon^2}{32n(S(\hat{f}))^2}}.\qedhere
\end{align*}
\end{proof}

\subsubsection{Bounding the type-1 difference}\label{sec:type1}
The following lemma provides tight upper bound on the type-1 difference.
\begin{lemma}
For $s$ satisfying $s-1, s, \text{ or } s+1 \in   nI^{**}_{j}$,
\[
\Abs{E_{\tilde{g}_j}(s)-E_{\tilde{g}_j}(s-1)}\leq \frac{\degree\cdot 2^{4.5\degree+1}}{n} \Lip_{g}^*(2n).
\]
\end{lemma}
\begin{proof}
Recall that
\[
h_{v,x_j}(s)=\sum_{l=0}^{v}\binom{v}{l}(-x_j)^{v-l}\prod_{l'=0}^{l-1}\Paren{\frac{s}{n}-\frac{l'}{n}}.
\]
The difference between $h_{v,x_j}(s)$ and  $h_{v,x_j}(s-1)$ is
\begin{align*}
h_{v,x_j}(s)-h_{v,x_j}(s-1)
&= {\sum_{l=0}^{v}\binom{v}{l}(-x_j)^{v-l}\Paren{\prod_{l'=0}^{l-1}\Paren{\frac{s}{n}-\frac{l'}{n}}-\prod_{l'=0}^{l-1}\Paren{\frac{s-1}{n}-\frac{l'}{n}}}}\\
&= {\sum_{l=0}^{v}\binom{v}{l}(-x_j)^{v-l}\Paren{\prod_{l'=0}^{l-1}\Paren{\frac{s}{n}-\frac{l'}{n}}-\prod_{l'=1}^{l}\Paren{\frac{s}{n}-\frac{l'}{n}}}}\\
&= {\sum_{l=0}^{v}\binom{v}{l}(-x_j)^{v-l}\Paren{\frac{s}{n}\prod_{l'=1}^{l-1}\Paren{\frac{s}{n}-\frac{l'}{n}}-\frac{s-l}{n}\prod_{l'=1}^{l-1}\Paren{\frac{s}{n}-\frac{l'}{n}}}}\\
&= {\sum_{l=0}^{v}\frac{l}{n}\binom{v}{l}(-x_j)^{v-l}{\prod_{l'=1}^{l-1}\Paren{\frac{s}{n}-\frac{l'}{n}}}}\\
&= \frac{v}{n}{\sum_{l=0}^{v-1}\binom{v-1}{l}(-x_j)^{(v-1)-l}{\prod_{l'=0}^{l-1}\Paren{\frac{s-1}{n}-\frac{l'}{n}}}}\\
&= \frac{v}{n}h_{v-1,x_j}(s-1).
\end{align*}
By Lemma~\ref{coeffbound} and the definition of $\Lip_{g}^*(2n)$,
\[
|a_{jv}|\leq 2^{3.5\degree} \cdot 2\cdot \sup_{z_1,z_2\in  I^{**}_j}|g(z_1)-g(z_2)|\leq  2^{3.5\degree+1} \Lip_{g}^*(2n) |I^{**}_j|. 
\]
Therefore, the quantity of interest satisfies
\begin{align*}
\Abs{E_{\tilde{g}_j}(s)-E_{\tilde{g}_j}(s-1)}
&=\Abs{\sum_{v=0}^{\degree}a_{jv}|I^{**}_j|^{-v} \Paren{h_{v,x_j}(s)-h_{v,x_j}(s-1)}}\\
&=\Abs{\sum_{v=0}^{\degree}a_{jv}|I^{**}_j|^{-v}  \frac{v}{n}h_{v-1,x_j}(s-1)}\\
&\leq \frac{2^{3.5\degree+1} }{n}\Lip_{g}^*(2n) |I^{**}_j|\sum_{v=1}^{\degree} v |I^{**}_j|^{-v}\Paren{2 |I^{**}_j|}^{v-1}\\
&\leq \frac{2^{3.5\degree+1} }{n}\Lip_{g}^*(2n)\sum_{v=1}^{\degree} v{2}^{v-1}\\
&\leq \frac{\degree\cdot 2^{4.5\degree+1} }{n}\Lip_{g}^*(2n),
\end{align*}
where the third last inequality follows from Lemma~\ref{boundg3} by setting $\delta=|I^{**}_j|$, and the last inequality follows from $\sum_{v=1}^{\degree} v{2}^{v-1}\leq \degree \cdot 2^\degree$.
\end{proof}

\subsubsection{Bounding the type-2 difference}\label{sec:type2}
In this section, we show the following upper bound on the type-2 difference.
\begin{lemma}
For $s$ satisfying $s\in   nI^{**}_{j-1}\cap nI^{**}_{j}$,
\[
\Abs{E_{\tilde{g}_{j-1}}(s)-E_{\tilde{g}_j}(s)}\leq \frac{4\cdot 2^{4.5\degree}}{n}  \Devstar_{g}(2n).
\]
\end{lemma}

\begin{proof}
Note that $E_{\tilde{g}_{j-1}}(N_i)-E_{\tilde{g}_j}(N_i)$ is an unbiased estimator of $\Paren{\tilde{g}_{j-1}-\tilde{g}_j}(x)$.
For simplicity, denote $\tilde{q}_j(x):=(\tilde{g}_{j-1}-\tilde{g}_j)(x)$ and ${I}^{\Lambda}_j:=  I^{**}_{j-1}\cap I^{**}_{j}=c_n \Br{(j-3)^2\indic_{j\geq 3}, (j+1)^2}$. Then we have $|\tilde{q}_j(x)|\leq 2{\Devstar_{g}(2n)}/{n}$ for $x\in {I}^{\Lambda}_j$.
Let
\[
x_j':=c_n (j-3)^2\indic_{j\geq 3}
\]
be the left end point of ${I}^{\Lambda}_j$, and 
\[
|{I}^{\Lambda}_j|:=c_n (j+1)^2-c_n (j-3)^2\indic_{j\geq 3}
\]
be the length of ${I}^{\Lambda}_j$.
For any $x\in{I}^{\Lambda}_j$, there exists $y_x\in[0,1]$ such that
\[
x=x_j'+|{I}^{\Lambda}_j| y_x.
\]
Since $x\rightarrow y_x$ is a linear transformation, there exist coefficients $b_{jv}, v=0,\ldots, \degree$, independent of $x$, such that
\[
\tilde{q}_j(x)=\sum_{v=0}^{\degree} b_{jv}y_x^{v}.
\]
By the definition of $\tilde{q}_j(x)$ and the triangle inequality,  we can deduce that $|\tilde{q}_j(x)|\leq 2{\Devstar_{g}(2n)}/{n}$ for all $x\in {I}^{\Lambda}_j$. Furthermore, according to Lemma~\ref{coeffbound},
for $v\ge 1$,
\[
|b_{jv}|\leq \frac{2^{4.5\degree}}{n} \Devstar_{g}(2n).
\]
Note that the bound also holds for $v=0$ since $|b_{j0}|=|\tilde{q}_j(x_j')|\leq 2{\Devstar_{g}(2n)}/{n}$.
In addition, substituting $y_x$ by ${|{I}^{\Lambda}_j|}^{-1}(x-x_j')$, we can rewrite $\tilde{q}_j(x)$ as
\[
\tilde{q}_j(x)=\sum_{v=0}^{\degree} b_{jv}{|{I}^{\Lambda}_j|}^{-v}\Paren{x-x_j'}^{v}.
\]
Consequently, we have the following equality:\vspace{-0.5em}
\[
\Paren{E_{\tilde{g}_{j-1}}-E_{\tilde{g}_j}}(s)=\sum_{v=0}^{\degree} b_{jv}{|{I}^{\Lambda}_j|}^{-v} h_{v,x_j'}(s).
\]
Therefore, for all $s\in n{I}^{\Lambda}_j$, 
\begin{align*}
\Abs{\Paren{E_{\tilde{g}_{j-1}}-E_{\tilde{g}_j}}(s)}
&=\Abs{\sum_{v=0}^{\degree} b_{jv}|{I}^{\Lambda}_j|^{-v} h_{v,x_j'}(s)}\\
&\leq {\sum_{v=0}^{\degree} \frac{2\cdot2^{3.5\degree} \Devstar_{g}(2n)}{n}|{I}^{\Lambda}_j|^{-v} \Paren{2|{I}^{\Lambda}_j|}^v}\\
&\leq \frac{2^{3.5\degree} \Devstar_{g}(2n)}{n} \sum_{v=0}^{\degree} 2^{v+1}\\
&\leq \frac{4\cdot 2^{4.5\degree}}{n} \Devstar_{g}(2n).\qedhere
\end{align*}
\end{proof}
\vspace{-1em}

\section{Proofs of other theorems}\label{corproof}
\subsection{Proof of Theorem 5}\label{cor2}

Let $\vec{p}\in \Delta_k$\ be an arbitrary distribution and $X^N$ be an \iid\ sample sequence from $\vec p$ of an independent $N\sim \Poi(2n)$ size. 
Applying sample splitting to $X^N$, we denote by $N_i$ and $N_i'$ the number of times symbol $i\in[k]$ appearing in the first and second sub-sample sequences, respectively. 
Applying the technique presented in Section~\ref{sec:estconst}, we can estimate the additive property
\[
f(\vec{p})=\sum_{i\in [k]} f_i(p_i)
\]
by the estimator
\[
\hat{f}^*(X^N):=\sum_{i\in[k]} \hat{f}_i^*(N_i,N_i').
\]
We start by bounding the bias of $\hat{f}^*$. Fix $\lambda\in (0,1/4)$ and let $T$ be a sufficiently large constant satisfying $T_1\gg \max_{i\in[k]}\max_{x\in [0,1]}|f_i(x)|$. The results in Section~\ref{sec:bias} and triangle inequality imply
\begin{align*}
\Abs{\EE[\hat{f}^*(X^N)]-f(\vec{p})}
&=\Abs{\EE\Br{\sum_{i\in[k]} \hat{f}_i^*(N_i,N_i')}-\sum_{i\in [k]} f_i(p_i)}
\leq \sum_{i\in [k]} \Abs{\EE[\hat{f}_i^*(N_i,N_i')]-f_i(p_i)}\\
&\leq \sum_{i\in [k]} \Paren{\frac{T}{n^5}\cdot p_i+\frac{3T\degree\cdot 2^{4.5\degree}}{c_n n^5}\cdot p_i+\frac{5}{n}\Devstar_{f_i}(2n,{p_i})}\\
&=\frac{T}{n^5}+\frac{3T\degree\cdot 2^{4.5\degree}}{c_n n^5}+\frac{5}{n}\sum_{i\in [k]}\Devstar_{f_i}(2n,{p_i})\\
&\leq \frac{T}{n^5}+\frac{Tn^\lambda}{c n^4\log n}+\frac{5}{n}\sum_{i\in [k]}\Devstar_{f_i}(2n,{p_i}).
\end{align*}
Next we analyze the variance of $\hat{f}^*$. Due to Poisson sampling and sample splitting, all the counts $N_i$ and $N_i'$, $i\in[k]$ are mutually independent. Therefore, by Lemma~\ref{lem:variance} in Section~\ref{sec:variance},
\begin{align*}
\Var(\hat{f}^*(X^N))
&=\Var\Paren{\sum_{i\in[k]} \hat{f}_i^*(N_i,N_i')}
=\sum_{i\in[k]} \Var\Paren{\hat{f}_i^*(N_i,N_i')}\\
&\leq \sum_{i\in [k]}\Paren{\frac{72c(\log n)}{n^{1-3\lambda}}\Paren{ \Lipstar_{f_i}(2n, {p_i})}^2\cdot p_i+\frac{8T^2}{n^5}\cdot p_i}\\
&=\frac{6T^2}{n^5}+\frac{2c(\log n)}{n^{1-3\lambda}}\sum_{i\in [k]}\Paren{36 \Lipstar_{f_i}(2n, {p_i})}^2\cdot p_i.
\end{align*}
To characterize higher-order central moments of $\hat{f}^*$, note that changing one sample point in $X^N$ would change the counts $N_i$, $N_i'$, or both for at most two symbols. Hence, according to Section~\ref{sec:conc}, for a given $n$ the sensitivity of $\hat{f}^*$, also defined in the same section, satisfies
\[
S(\hat{f}^*)\leq \frac{4\max_{i\in[k]}S_{f_i}^*(2n)}{n^{1-\lambda}}.
\]
This bound together with Lemma~\ref{trueMcD} yields
\[
\Pr\Paren{\Abs{\hat{f}^*(X^{N})-\EE\Br{\hat{f}^*(X^{N})}}>\varepsilon}
\leq4\exp\Paren{-\frac{n^{1-2\lambda}\varepsilon^2}{(32\max_{i\in[k]}S_{f_i}^*(2n))^2}}.
\]

\subsection{Proof of Theorem 1}\label{cor5}
Recall that an additive property $f$ is a Lipschitz property if all the $f_i$'s have uniformly bounded Lipschitz constants. 
Our proof of Theorem 1 relies on the following lemma, which corresponds to Theorem 7.2 in~\cite{lipbound} whose proof is completely constructive. In other words, there is an explicit procedure to compute the polynomial described in the following lemma. 
\begin{lemma}\label{lipb}
There exists a universal constant $C$ such that for any degree parameter $d\in\mathbb{Z}$ and $1$-Lipschitz function $g$ over an arbitrary bounded interval $I:=[x_1,x_2]$, one can find a polynomial $\tilde{g}$ of degree at most $d$ satisfying
\[
 |\tilde{g}(x)-g(x)|\leq \frac{C\sqrt{|I|(x-x_1)}}{d}, \forall x\in I.
\]
\end{lemma}
We restate Theorem 1 below under Poisson sampling. By the results in~\cite{yihongentro}, this suffices to imply the corresponding result under fixed sampling, where the sample size is fixed to be $n$. 
\setcounter{theorem}{0}
\begin{theorem}\label{cor:Lipschitz}
If $f$ is an $L$-Lipschitz property, then for any $\vec p\in \Delta_k$, $N\sim \Poi(n)$, and $X^N\sim \vec p$,
\[
\Abs{\EE\Br{\hat{f}^*(X^{N})}-f(\vec{p})}\lesssim\sum_{i\in[k]}L\sqrt{\frac{p_i}{n\log n}}\leq L\sqrt{\frac{k}{n\log n}},
\]
and
\[
\Var(\hat{f}^*(X^{N}))\leq \frac{L^2}{n^{1-4\lambda}}.
\]
\end{theorem}
\begin{proof}
Without loss of generality, we assume that all the $f_i$'s have Lipschitz constants uniformly bounded by $1$. 
The derivations in Section~\ref{sec:bias} and~\ref{cor2} imply
\begin{align*}
\Abs{\EE[\hat{f}^*(X^N)]-f(\vec{p})}
&\lesssim\frac{1}{n^3}+\sum_{i\in[k]}\max_{j'\in[j_{p_i}-1,j_{p_i}+1]}|\tilde{f}_{i,j'}(p_i)-f(p_i)|,
\end{align*}
Here, for $j'>3$, we choose $\tilde{f}_{i,j'}(x)$ to be the min-max polynomial defined in Section~\ref{sec:estconst}; for $j'\le 3$, we employ the polynomials used in Lemma~\ref{lipb} instead.
Note that the latter polynomials may not be
the min-max polynomials. However, this would not affect our analysis as our proof in Section~\ref{sec:prop} also holds for these polynomials (simply change the definition of $\Devstar_{g}(2n,{p})$). 

For any symbol $i$ satisfying $j_{p_i}\le 3$, 
\[
\max_{j'\in[j_{p_i}-1,j_{p_i}+1]}|\tilde{f}_{i,j'}(p_i)-f_i(p_i)|
\lesssim \max_{j'\in[j_{p_i}-1,j_{p_i}+1]} \frac{\sqrt{|I_{j'}|(p_i-0)}}{d_n}
\asymp  \frac{\sqrt{\frac{log n}{n}p_i}}{\log n}
= \sqrt{\frac{p_i}{n\log n}}.
\]
On the other hand, applying Lemma~\ref{lipb} and the definition of min-max polynomials to our case implies that for any symbol $i$ satisfying $j_{p_i}> 3$, 
\[
\max_{j'\in[j_{p_i}-1,j_{p_i}+1]}|\tilde{f}_{i,j'}(p_i)-f_i(p_i)|
\lesssim \max_{j'\in[j_{p_i}-1,j_{p_i}+1]} \frac{|I_{j'}|}{d_n}
\asymp  \frac{ j_{p_i}}{n}
\]
and
\[
p_i\in I^{**}_{p_i}=c\frac{\log n}{n}[(j_{p_i}-3)^2, (j_{p_i}+2)^2],
\]
or equivalently,
\[
j_{p_i}
\in \Br{\sqrt{\frac{np_i}{c\log n}}-2, \sqrt{\frac{np_i}{c\log n}}+3}
\subseteq \Br{\sqrt{\frac{np_i}{c\log n}}-2, 4\sqrt{\frac{np_i}{c\log n}}}.
\]
Therefore, 
\[
\max_{j'\in[j_{p_i}-1,j_{p_i}+1]}|\tilde{f}_{i,j'}(p_i)-f_i(p_i)|
\lesssim \frac{ j_{p_i}}{n}
\le \frac{1}{n}\cdot 4\sqrt{\frac{np_i}{c\log n}}
\lesssim \sqrt{\frac{p_i}{n\log n}}.
\]
The above result together with the Cauchy-Schwarz inequality implies
\[
\Abs{\EE[\hat{f}^*(X^N)]-f(\vec{p})}
\lesssim\frac{1}{n^3}+\sum_{i\in[k]}\sqrt{\frac{p_i}{n\log n}}\leq 2\sum_{i\in[k]}\sqrt{\frac{p_i}{n\log n}}\leq 2\sqrt{\frac{k}{n\log n}},
\]
where the second inequality follows by observing $\sqrt{a+b}\leq \sqrt{a}+\sqrt{b}$. Analogously, by previous results, we can bound the variance of $\hat f^*$ as follows,
\[
\Var(\hat{f}^*(X^N))
\lesssim\frac{1}{n^5}+\frac{\log n}{n^{1-3\lambda}}\sum_{i\in [k]}\Paren{ \Lipstar_{f_i}(2n, {p_i})}^2\cdot p_i.
\]
By the definition of $\Lipstar_{f_i}$ and the assumption that $f_i$ is $1$-Lispchitz, 
\[
\Lipstar_{f_i}(2n, {p_i})=\max_{j'\in[j_{p_i}-1,j_{p_i}+1]}\Brace{\sup_{x,y\in I_{j'}^{**}, |y-x|\ge 1/n} \frac{\Abs{f_i\Paren{y}-f_i\Paren{x}}}{|y-x|}}\leq 1.
\]
Hence, 
\[
\Var(\hat{f}^*(X^N))
\lesssim \frac{1}{n^5}+\frac{\log n}{n^{1-3\lambda}}\sum_{i\in [k]} p_i
\le \frac{1}{n^5}+\frac{\log n}{n^{1-3\lambda}}
\leq \frac{1}{n^{1-4\lambda}}.\qedhere
\]
\end{proof}

\subsection{Private property estimation}\label{sec:dp}
According to~\citep{priv_prop}, we can construct a differentially private property estimator $\hat{f}^*_{\textit{DP}}$ by first applying $\hat{f}^*$ to the sample sequence, and then privatizing its estimate through adding Laplace noise. The following lemma characterizes the sample complexity of $\hat{f}^*_{\textit{DP}}$, and enables us to upper bound the private sample complexity of estimating general additive properties.
\begin{lemma}\label{lem:dp}
There exists a universal constant $c^*$ such that
\[
C_f(\hat{f}^*_{\text{DP}},2\varepsilon,1/3,2\alpha) \leq \frac{c^*}{4}\left\{C_f(\hat{f}^*,\varepsilon,{1}/{3})+\min\Brace{n:\ S(\hat{f}^*, n)\leq \varepsilon\alpha}\right\}.
\]
The right-hand side also upperly bounds $C_f(2\varepsilon,1/3, 2\alpha)$.
\end{lemma}

By Theorem~\ref{cor:sample_guarantees}, for $\hat{f}^*(X^n)$ to achieve an accuracy of $\varepsilon$ with probability at least $2/3$, for all $\vec{p}\in \Delta_k$, it suffices for the sampling parameter $n$ to satisfy the following three conditions:
\[
n\ge \Paren{\frac{4T}{\varepsilon}}^{\frac13}\!\!\!\!,\ \frac{n}{\Devstar_{f_i}(n)}\geq \frac{20k}{\varepsilon}, \text{ and }\ \frac{n^{\frac{1}{2}-\lambda}}{\Sensitivity_{f_i}(n)}\geq \frac{16\sqrt{\log12}}{\varepsilon}\ ,\forall i.
\]
where $T$ is a uniform upper bound on $|f_i(x)|,\forall i\in [k], x\in[0,1]$. To further make the estimator's $n$-sensitivity smaller than $\alpha\varepsilon$, the sampling parameter $n$ should also satisfy the condition:
\[
\frac{n^{1-\lambda}}{\Sensitivity_{f_i}(n)}\geq \frac{1}{\alpha\varepsilon}\ ,\forall i.
\]
Define $n_f(2\varepsilon, 2\alpha)$ as the smallest $n$ satisfying all the four inequalities above. Then,
\setcounter{theorem}{5}
\begin{theorem}
The $(\varepsilon, 1/3, \alpha)$-private sample complexity for 
any additive property $f$ satisfies
\[
C_f(\varepsilon,1/3,\alpha) \le n_f(\varepsilon, \alpha).
\]
\end{theorem}
\subsection{High-probability property estimation}\label{sec:highprob}
In this section, we present tight upper and lower bounds on the $(\varepsilon,\delta)$-sample complexity of estimating various properties. The error parameter $\varepsilon$ can take any value in $(0,1)$. All the upper bounds follow from Theorem~\ref{cor:sample_guarantees}. Below we focus on deriving the lower bounds. 

\subsubsection*{Shannon entropy}
For any absolute constant $\beta\in(0,1)$,
\[
\frac{k}{\varepsilon \log k}+\log \frac{1}{\delta}\cdot\frac{\log^2 k}{\varepsilon^2}\lesssim C_{f}(\varepsilon,\delta).
\]
The first part of the lower bound follows directly from~\citep{yihongentro}. To show the second part of the lower bound, let $\varepsilon'\in(0,1)$ be a parameter to be determined later, and consider the following~\citep{yihongentro}  two distributions in $\Delta_k$,
\[
\vec p_1:=\Paren{\frac{1-\varepsilon'}{3(k-1)},\ldots,\frac{1-\varepsilon'}{3(k-1)},\frac{2+\varepsilon'}{3}}
\]
and 
\[
\vec p_2:=\Paren{\frac{1}{3(k-1)},\ldots,\frac{1}{3(k-1)},\frac{2}{3}}.
\]
The entropy difference between these two distributions is
\begin{align*}
H(\vec p_2)-H(\vec p_1)
&=\frac{1-\varepsilon'}{3}\log \frac{1-\varepsilon'}{3(k-1)}+{\frac{2+\varepsilon'}{3}}\log {\frac{2+\varepsilon'}{3}}-\frac{1}{3}\log \frac{1}{3(k-1)}-\frac{2}{3}\log {\frac{2}{3}}\\
&=\frac{\varepsilon'}{3}\log (2(k-1)) + \frac{1-\varepsilon'}{3}\log \Paren{1-\varepsilon'}+{\frac{2+\varepsilon'}{3}}\log {\frac{2+\varepsilon'}{2}}\\
&\geq \frac{\varepsilon'}{3}\log (2e^{-1}(k-1)).
\end{align*}
For sufficiently large $k$, choose $\varepsilon'={9\varepsilon}/{\log (2e^{-1}(k-1))}$. The difference between $H(\vec p_1)$ and $H(\vec p_2)$ is at least $3\varepsilon$. 

On one hand, since $\vec p_1$ and $\vec p_2$ differ by $2\varepsilon'/3$ in $\ell_1$ distance, any algorithm that distinguishes the two distributions with confidence $1-\delta$ requires at least $\Omega(\frac{1}{\varepsilon'^2}\log\frac{1}{\delta})$ samples.\vspace{-0.2em} On the other hand, any entropy estimator $\hat f$ that utilizes $C_f(\hat{f},\varepsilon, \delta)$ samples can be used to distinguish $\vec p_1$ and $\vec p_2$ with confidence $1-\delta$. This yields the desired lower bound.

\subsubsection*{Normalized support size} 
The lower bound follows from~\citep{yihongsupp}. 

\subsubsection*{Power sum}
For any absolute constants $\beta\in(0,1)$ and $a\in(1/2,1)$,
\[
\frac{k^{\frac{1}{a}}}{\varepsilon^{\frac{1}{a}} \log k}+\log \frac{1}{\delta}\cdot\frac{k^{2-2a}}{\varepsilon^2}
\lesssim 
C_{f}(\varepsilon,\delta)
\]
and
\[
C_{f}(\varepsilon,\delta)
\lesssim
\frac{k^{\frac{1}{a}}}{\varepsilon^{\frac{1}{a}} \log k}+\Br{\Paren{\log \frac{1}{\delta}\cdot\frac{1}{\varepsilon^2}}^{\frac{1}{2a-1}}}^{1+\beta}.
\]
The first part of the lower bound follows from~\cite{jiaoentro}. Analogously, to show the second part of the lower bound, let $\varepsilon''\in(0,1)$ be a parameter to be determined later, and consider the following two distributions in $\Delta_k$,
\[
\vec p_3:=\Paren{\frac{1-\varepsilon''}{3(k-1)},\ldots,\frac{1-\varepsilon''}{3(k-1)},\frac{2+\varepsilon''}{3}}
\]
and 
\[
\vec p_2:=\Paren{\frac{1}{3(k-1)},\ldots,\frac{1}{3(k-1)},\frac{2}{3}}.
\]
The difference between the power sums of these two distributions satisfies
\begin{align*}
P_a(\vec p_2)-P_a(\vec p_3)
&=(k-1)\Paren{\frac{1}{3(k-1)}}^a+\Paren{\frac{2}{3}}^a-(k-1)\Paren{\frac{1-\varepsilon''}{3(k-1)}}^a-\Paren{\frac{2+\varepsilon''}{3}}^a\\
&=\frac{(k-1)^{1-a}}{3^a}\Paren{1-(1-\varepsilon'')^a}+\Paren{\frac{2}{3}}^a-\Paren{\frac{2+\varepsilon''}{3}}^a\\
&\geq \frac{(k-1)^{1-a}}{3^a}a\Paren{\varepsilon''-\frac{\varepsilon''^2}{2}}+\Paren{\frac{2}{3}}^a\Paren{1-a\Paren{1+\frac{\varepsilon''}{2}}}\\
&\geq \frac{a\varepsilon''}{2\cdot 3^a}\Paren{(k-1)^{1-a}-2^a}.
\end{align*}
For $k$ that is sufficiently large, choose parameter $\varepsilon''={6\varepsilon \cdot 3^a}/\Paren{a(k-1)^{1-a}-a\cdot2^a}$. The difference between $P_a(\vec p_2)$ and $P_a(\vec p_3)$ is at least $3\varepsilon$. 

The desired lower bound follows from the same reasoning as in the Shannon-entropy case.

\bibliography{est_ref}

\end{document}